%% file: draft3.tex
\documentclass[preprint,11pt]{article}

\usepackage{amssymb,amsfonts,amsmath,amsthm,amscd,dsfont,mathrsfs}
\usepackage{graphicx,float,psfrag,epsfig,cite}
\usepackage{wrapfig}
\usepackage{relsize}
\usepackage{color}
\usepackage{hyperref}
\usepackage[tight]{subfigure}
\usepackage{algorithm}
\usepackage[noend]{algorithmic}
\usepackage{caption}

\usepackage{tikz}
\usetikzlibrary{calc,shapes}

\def\Ebb{{\mathbb E}}

\usepackage{authblk}
\usepackage{multirow}

\newcommand{\aacomment}[1]{}

\DeclareMathAlphabet{\mathpzc}{OT1}{pzc}{m}{it}

\newcommand\E{\mathbb{E}}
\newcommand\R{\mathbb{R}}
\newcommand\T{{\scriptscriptstyle\top}}

\newcommand\veps{\varepsilon}
\newcommand\Dir{\operatorname{Dir}}

\newcommand\diag{\operatorname{diag}}
\newcommand\supp{\operatorname{supp}}

\newcommand\Part{\mathcal{P}}

\newcommand\rank{\operatorname{rank}}

\newcommand\Col{\operatorname{Col}}
\newcommand\Row{\operatorname{Row}}
\def \G {\mathcal{G}}
\def \PA {{\sf PA}}
\def \P {\mathbb{P}}
\newcommand\V{\ensuremath{\mathcal{V}}}
\newcommand\vobs{\ensuremath{\V_{\operatorname{obs}}}}
\newcommand\vhid{\ensuremath{\V_{\operatorname{hid}}}}
\newcommand\edges{\ensuremath{\mathcal{E}}}
\def \H{\mathcal{H}}
\def \Neigh{{\rm N}}
\def\AlgI{{\textsc{TWMLearn}}}
\def\AlgII{{\textsc{TMLearn}}}
\def \Subroutine{{\textsc{DLD}}}
\def\AlgIproj {{\textsc{TWMLearn(proj)}}}
\def \nsmp {n_{{\rm smp}}}

\def \expo{{\sf Exp}}
\def\Poiss{{\sf Poisson}}
\def\normal{\sf{N}}

\def \hSigma {\widehat{\Sigma}}
\def\dist{{\sf dist}}
\def\Pairs{{\rm Pairs}}
\def\Triples {{\rm Triples}}

\newtheorem{proposition}{Proposition}[section]
\newtheorem{lemma}[proposition]{Lemma}
\newtheorem{definition}[proposition]{Definition}

\newtheorem{theorem}[proposition]{Theorem}
\newtheorem{claim}[proposition]{Claim}

\newtheorem{remark}[proposition]{Remark}
\newtheorem{condition}{Condition}

\usepackage{color}
\newcommand{\djh}[1]{}

\newcommand{\aj}[1]{}

\title{Learning    Topic Models and Latent Bayesian Networks Under Expansion Constraints}

\author[1]{Animashree Anandkumar}
\author[2]{Daniel Hsu}
\author[3]{Adel Javanmard}
\author[2]{Sham M. Kakade}
\affil[1]{Department of EECS, University of California, Irvine}
\affil[2]{Microsoft Research New England}
\affil[3]{Department of Electrical Engineering, Stanford University}

\begin{document}

\maketitle
{\def\thefootnote{}
\footnotetext{E-mail:
\texttt{a.anandkumar@uci.edu},
\texttt{dahsu@microsoft.com},
\texttt{adelj@stanford.edu},
\texttt{skakade@microsoft.com}}}

\begin{abstract}
Unsupervised estimation of latent variable models is a fundamental problem
central to numerous applications of machine learning and statistics.
This work presents a principled approach for estimating broad classes of
such models, including probabilistic topic models and latent linear
Bayesian networks, using only second-order observed moments.
The sufficient conditions for identifiability of these models are primarily
based on weak expansion constraints on the topic-word matrix, for topic
models, and on the directed acyclic graph, for Bayesian networks.
Because no assumptions are made on the distribution among the latent
variables, the approach can handle arbitrary correlations among the topics
or latent factors.
In addition, a tractable learning method via $\ell_1$ optimization is
proposed and studied in numerical experiments.
%
\end{abstract}


\input{intro-DAG-3}

\input{summary-3}
\input{related}
%
 \input{model-results}

\input{main-results}

\input{BN}
\input{single-view}

\input{matrix-decomposition}
\input{hierarchical}
\input{simulation-dag}

\section*{Acknowledgements}
We thank David Gamarnik and Rong Ge for helpful discussions. A.~Anandkumar acknowledges the support of NSF Career Award CCF-1254106, NSF Award CCF 1219234, AFOSR Award FA9550-10-1-0310, and ARO Award W911NF-12-1-0404.
Part of this work was completed while A.~Anandkumar and A.~Javanmard were
visiting Microsoft Research New England.

\bibliographystyle{abbrv}
\bibliography{all-bibliography}


\input{proofs}

\end{document}

%% file: intro-DAG-3.tex
\section{Introduction}
%


%

It is widely recognized that incorporating latent or hidden variables is a crucial aspect of modeling.
Latent variables can provide a succinct representation of the observed data through dimensionality reduction; the possibly many observed variables are summarized by fewer hidden
effects.
Further, they are central to  predicting causal relationships and interpreting the hidden effects as
unobservable concepts.
For instance in sociology, human behavior is affected by abstract
notions such as social attitudes, beliefs, goals and plans.
As another example, medical knowledge is organized into casual hierarchies
of invading organisms, physical disorders, pathological states and symptoms,
and only the symptoms are observed.

In addition to incorporating latent variables, it is also important to model the complex dependencies among the variables. A popular class of models for incorporating such dependencies are the
Bayesian networks, also known as  belief networks. They incorporate a set of causal and conditional independence relationships  through  directed acyclic graphs (DAG)~\cite{Pearl:book}. They have  widespread applicability in artificial intelligence~\cite{Koller:SRL07,Boutilier:UAI96, Koller:UAI97, Choi-10}, in the social sciences~\cite{Bagozzi-80,Kohn-82,Wheaton-78,Bollen-89, Pearl-00, Pearl-soc-98}, and as structural equation models in economics~\cite{Awokuse-03, Haavelmo-43,Zellner-71,Bollen-89, Pearl-00, Spirtes-05}.

An important statistical task is to learn  such latent Bayesian networks from observed data. This involves discovery of the hidden variables, structure estimation (of the DAG) and estimation of the model parameters. Typically, in the presence of hidden variables, the learning task suffers from identifiability issues since there may be many models which can explain the observed data.  In order to overcome indeterminacy issues, one must restrict the set of
possible models.
We establish novel criteria for identifiability of latent DAG models using only   low order observed moments (second/third moments). We introduce a graphical constraint which we refer to as the \emph{expansion property} on the DAG.
Roughly speaking, expansion property states that every subset of hidden
nodes has ``enough'' number of outgoing edges in the DAG, so they have a
noticeable influence on the observed nodes, and thus on the samples drawn
from the joint distribution of the observed nodes. This notion implies new identifiability and learning results for DAG structures.

Another class of popular latent variable models are the probabilistic topic models~\cite{blei2012probabilistic}. In topic models, the latent variables correspond to the  topics in a document which generate
the (observed)  words. Perhaps, the most widely employed topic model is  the latent Dirichlet allocation (LDA)~\cite{blei2003latent}, which posits that the hidden topics are drawn from a Dirichlet distribution. Recent approaches have established that the LDA model can be learned efficiently using low-order (second and third) moments, using spectral   techniques~\cite{Anima-SVD-12,AGHKT}.
The LDA model, however, cannot incorporate arbitrary correlations\footnote{LDA models incorporate only ``weak'' correlations among   topics, since the Dirichlet distribution can be expressed as the set of independently distributed Gamma random variables, normalized by their sum: if $y _i \sim \Gamma(\alpha_i,1)$, we have  $(\frac{y_1}{\sum_i y_i}, \frac{y_2}{\sum_i y_i}, \ldots)   \sim \Dir(\alpha)$.} among the latent topics, and various correlated topic models have demonstrated superior empirical performance, e.g.~\cite{blei2007correlated,li2006pachinko}, compared to LDA. However, learning correlated topic models is challenging, and further constraints need to be imposed to establish identifiability and provable learning.

A typical (exchangeable) topic model is parameterized by the topic-word matrix, i.e., the conditional distributions of the words given the topics, and the latent topic distribution, which determines the mixture of topics in a document. In this paper, we allow for arbitrary (non-degenerate) latent topic distributions, but impose 
expansion  constraints on the topic-word matrix. In other words, the word support of different topics are not ``too similar'', which is a reasonable assumption. 
Thus, we establish expansion as an unifying criterion for guaranteed learning of both latent Bayesian networks and topic models.

%% file: summary-3.tex


\subsection{Summary of contributions}
\label{subsec:contribution}

We establish identifiability for different classes of topic models and latent Bayesian networks, and more generally, for linear latent models, and also propose efficient algorithms for the learning task.

\subsubsection{Learning Topic Models}

\paragraph{Learning  under expansion conditions. }
We adopt a moment-based approach to learning topic models, and specifically, employ second-order observed moments, which can be efficiently estimated using a small number of samples.
We establish identifiability of the topic models for arbitrary
(non-degenerate) topic mixture distributions, under assumptions on the
topic-word matrix. The support of the topic-word matrix   is a bipartite
graph which relates the topics to words. We impose a   weak (additive)
expansion constraint on this bipartite graph. Specifically, let $A\in
\R^{n\times k}$ denote the topic-word matrix, and  for any subset of topics
$S\subset [k]$ (\emph{i.e.}, a subset of columns of $A$), let $\Neigh(S)$
denote the set of neighboring words,  \emph{i.e.}, the set of words, the topics in $S$ are supported on.  We require that  \begin{equation}\label{condsummary:expansion}|\Neigh(S)| \geq |S| + d_{\max},\end{equation} where $d_{\max}$ is the maximum degree for any topic.
Intuitively, our expansion property states that every subset of
topics generates sufficient number of words.
We establish that under the above expansion condition in \eqref{condsummary:expansion}, for  generic\footnote{The precise definition for parameter genericity  is given in Condition~\ref{cond:parameter}.} parameters (for  non-zero entries of $A$), the columns of $A$ are the sparsest vectors in the column span, and are therefore, identifiable.

In contrast, note that for all subsets of topics $S\subset [k]$,   the condition $|\Neigh(S)| \geq |S|,$ is {\em necessary} for non-degeneracy of   $A$, and therefore, for identifiability of the topic model from second order observed moments. This implies that our sufficient condition in \eqref{condsummary:expansion} is close to the necessary condition for identifiability of sparse models, where the maximum degree of any topic $d_{\max}$ is small. Thus, we prove identifiability of topic models  under nearly tight expansion conditions on the topic-word matrix. Since the columns of $A$ are the sparsest vectors in the column span under~\eqref{condsummary:expansion}, this also implies recovery of $A$ through exhaustive search.
In addition, we  establish that the topic-word matrix can be learned efficiently through $\ell_1$ optimization, under some (stronger) conditions on the non-zero entries of the topic-word matrix, in addition to the expansion condition in~\eqref{condsummary:expansion}.
We call our algorithm $\AlgI$ as it learns the topic-word matrix.

\paragraph{Bayesian networks to model topic mixtures. }
The above framework does not impose any parametric assumption
on the distribution of the topic mixture $h$ (other than
non-degeneracy), and employs second-order observed moments
to  learn the topic-word matrix $A$ and the second-order
moments of $h$. If $h$ obeys a multivariate Gaussian
distribution, then this completely characterizes the
topic model. However, for general topic mixtures,
this is not sufficient to characterize the distribution of $h$,
 and further assumptions need to be imposed.
 A natural framework for modeling topic dependencies is via Bayesian networks~\cite{Lauritzen-96}. Moreover,
incorporating Bayesian networks for topic modeling also leads to efficient  approximate inference through belief propagation and their variants~\cite{Wainwright-08}, which have shown good empirical performance on sparse graphs.

 We consider the case where the latent topics can be modeled
 by a \emph{linear} Bayesian network, and establish that such networks
 can be learned efficiently
 using second and third order observed moments through a combination of
 $\ell_1$ optimization and spectral techniques.
The proposed algorithm is called $\AlgII$ as it learns (correlated) topic models.

\subsubsection{Learning (Single-View) Latent Linear Bayesian Networks}

The above techniques for learning topic models are also applicable for learning latent linear models, which includes linear Bayesian networks discussed in the introduction. This is because our method relies on the presence of a linear map from hidden to observed variables. In case of the topic models, the topic-word matrix represents the linear map, while for linear Bayesian networks, the (weighted) DAG from hidden to observed variables is the linear map.
Linear latent models are prevalent in a number of applications such as
blind deconvolution of sound and images~\cite{levin2009understanding}. The
popular independent component analysis (ICA)~\cite{ICAbook} is a special
case of our framework, where the sources (\emph{i.e.}, the hidden variables)   are assumed to be independent. In contrast, we allow for general latent distributions, and impose expansion conditions on the linear map from hidden to observed variables.

One key difference between topic models and other linear models (including
linear Bayesian networks) is that topic models are multi-view (\emph{i.e.}, have multiple words in the same document), while, for general linear models,  multiple views  may not be available. We require additional assumptions to provide recovery in the single-view setting.
We prove recovery under  certain rank conditions: we require that $n \geq 3k$, where $n$ is the dimension of the observed random vector and $k$, the dimension of the latent  vector, and the existence of a partition into three sets each with full column rank. Under these conditions,
we propose simple  matrix decomposition techniques to first ``de-noise'' the observed moments. These denoised moments are of the same form as the moments obtained from a topic model and thus, the techniques described for learning topic models can be applied on denoised moments.   
  Thus, we provide a general framework for guaranteed learning of linear latent models under expansion conditions.

\smallskip\noindent
\emph{Hierarchical topic models.} An important application of these techniques is in learning hierarchical  linear models,  where the developed method can be applied recursively, and the estimated second order moment of each layer  can be employed to further learn the deeper layers.
See Fig.~\ref{fig:multi-level} for an illustration.

\begin{figure}[!t]
\centering
\subfigure[Hierarchical topic model]{
\includegraphics*[viewport = 160 80 570 480, width=2.1in]{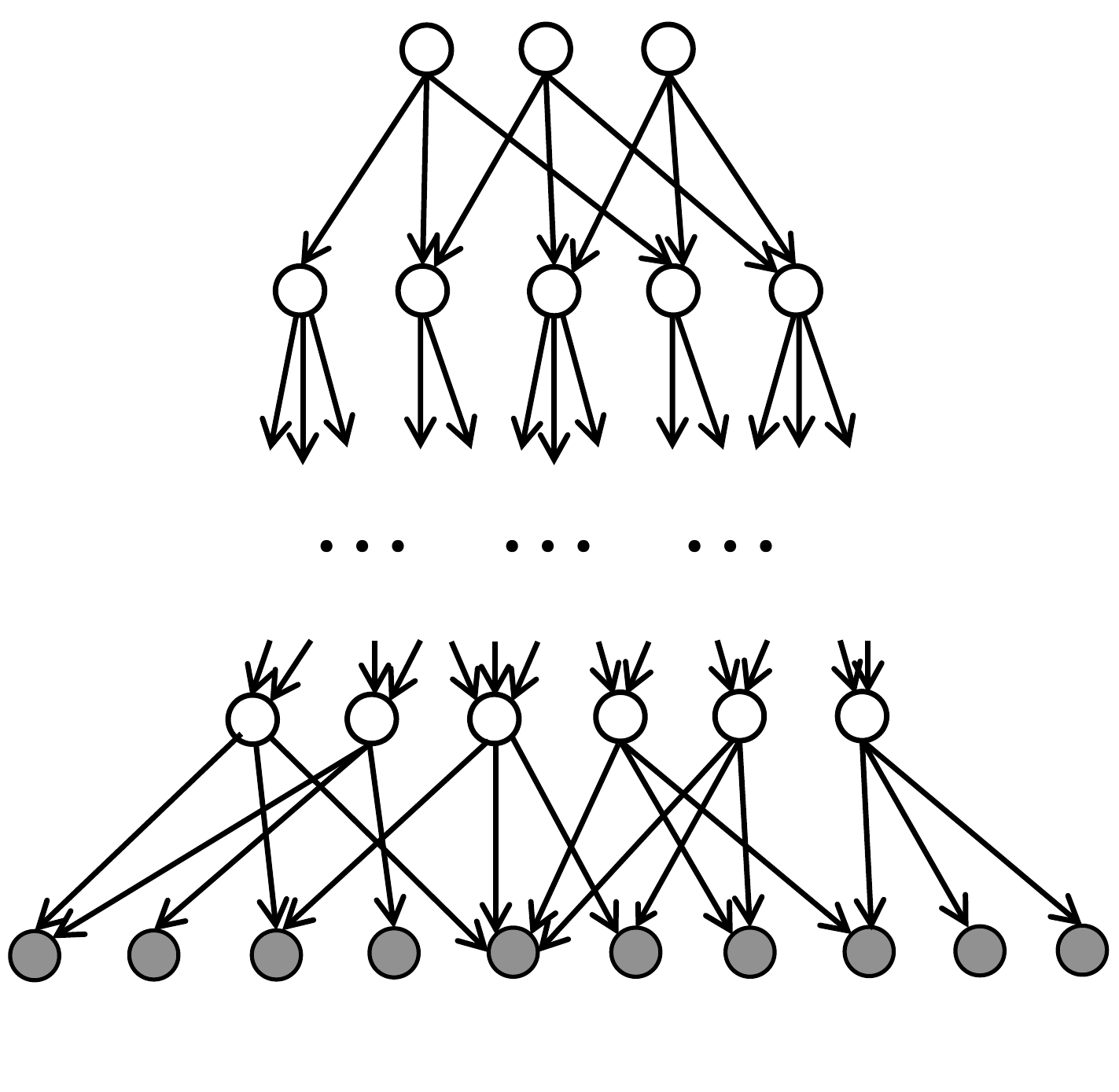}\label{fig:multi-level}
}
\hspace{0.5cm}
\subfigure[Bayesian networks to model topic mixtures]{
\includegraphics*[viewport = 110 175 600 440, width=3.4in]{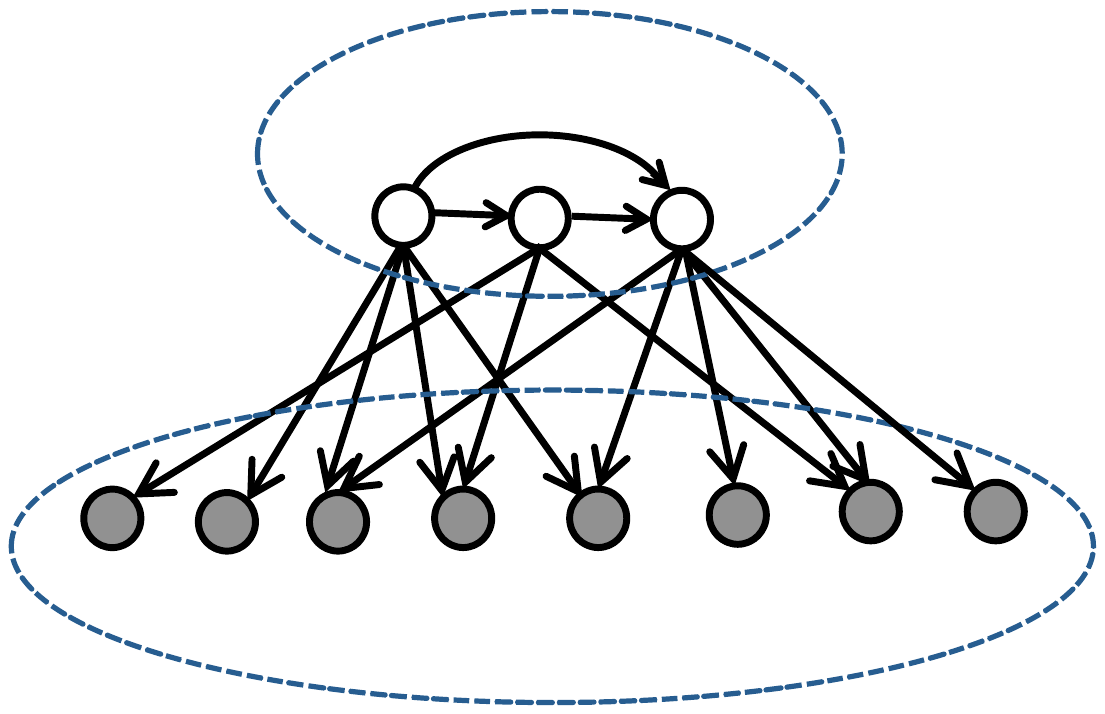}\label{fig:topic-model}
\put(-121,135){{\small Topics}}
\put(-133,118){\scriptsize $h(1)$}
\put(-114,118){\scriptsize $h(2)$}
\put(-94,118){\scriptsize $h(3)$}
\put(-179,45){\scriptsize $x(1)$}
\put(-162,45){\scriptsize $x(2)$}
\put(-145,45){\scriptsize $x(3)$}
\put(-127,45){\scriptsize $x(4)$}
\put(-108,45){\scriptsize $x(5)$}
\put(-88,45){\scriptsize $x(6)$}
\put(-70,45){\scriptsize $x(7)$}
\put(-50,45){\scriptsize $x(8)$}
\put(-160,80){$A$}
\put(-157,18){{\small Words in the vocabulary}}
}
%
\caption{Illustrations of hierarchical topic models and
Bayesian networks for topic mixtures. Words and topics are respectively shown by shaded and
white circles. Under the expansion property for the graph, we prove identifiability of these models from low order moments of the words. }
\end{figure}

%

\paragraph{Examples of graphs which can be learned. }
It is useful to consider some concrete examples
which satisfy the expansion property in~\eqref{condsummary:expansion}:

\smallskip\noindent
\emph{Full $d$-regular trees.}
These are tree structures in which every node other than the leaves has $d$ children. These are included in the ensemble of hierarchical models. We see that for $d \ge 2$, the model satisfies the expansion condition~\eqref{condsummary:expansion}, but require $d \ge 3$ to satisfy the rank condition. See Fig.~\ref{fig:ternary-tree} for an illustration of a full ternary tree with latent variables.

\smallskip\noindent
\emph{Caterpillar trees.}
These are tree structures in which all the leaves are within distance one of a central path. See Fig.~\ref{fig:caterpillar} for an illustration. These structures  have effective depth one. Let $d_{\max}$ and $d_{\min}$ respectively denote the maximum and the minimum number of leaves connected to a  fixed node on the central path. It is immediate to see that if $d_{\min} \ge d_{\max}/2+1$, the structure has the expansion property in~\eqref{condsummary:expansion}.

\smallskip\noindent
\emph{Random bipartite graphs.}
Consider bipartite graphs with hidden nodes in one part and observed nodes
in the other part. Each edge (between the two parts) is included in the
graph with probability $\theta$, independent from every other edge. It is
easy to see that, for any set $S \subseteq [k]$, the expected number of its
neighbors is : $\E|\Neigh(S)| = n(1-(1-\theta)^{|S|})$. Also, the expected
degree of the hidden nodes is $\theta n$. Now, by applying a Chernoff
bound, one can show that these graphs have the expansion property with high
probability, if $1-\sqrt{1-2k/n} < \theta <1/2$, \emph{i.e.}, with probability converging to one as $n \to \infty$.


\begin{figure}[!t]
\centering
\subfigure[Full ternary tree]{
\includegraphics*[viewport = 160 190 630 480, width=2.5in]{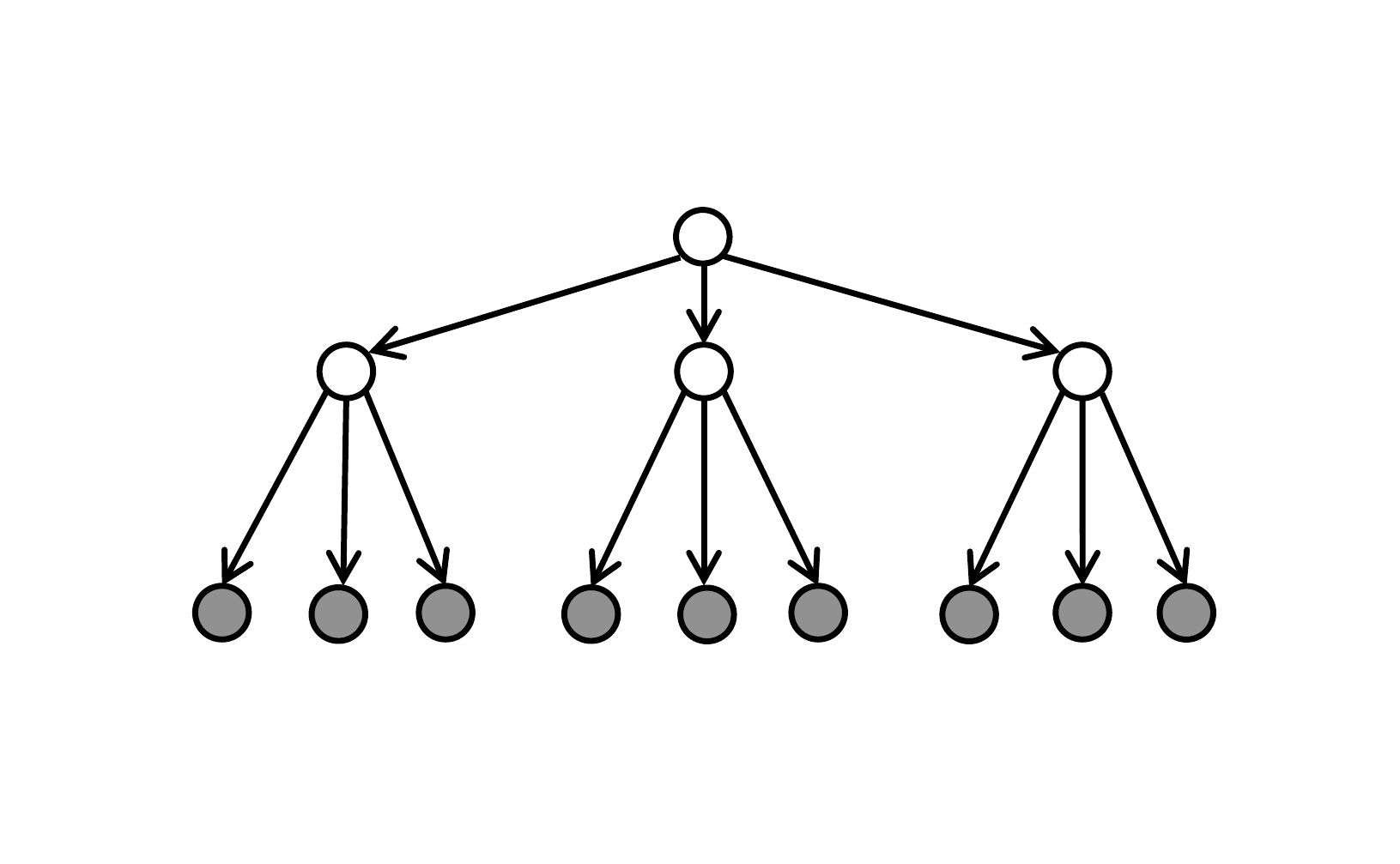}\label{fig:ternary-tree}
}
\hspace{1cm}
\subfigure[Caterpillar tree]{
\includegraphics*[viewport = 90 140 670 470, width=2.5in]{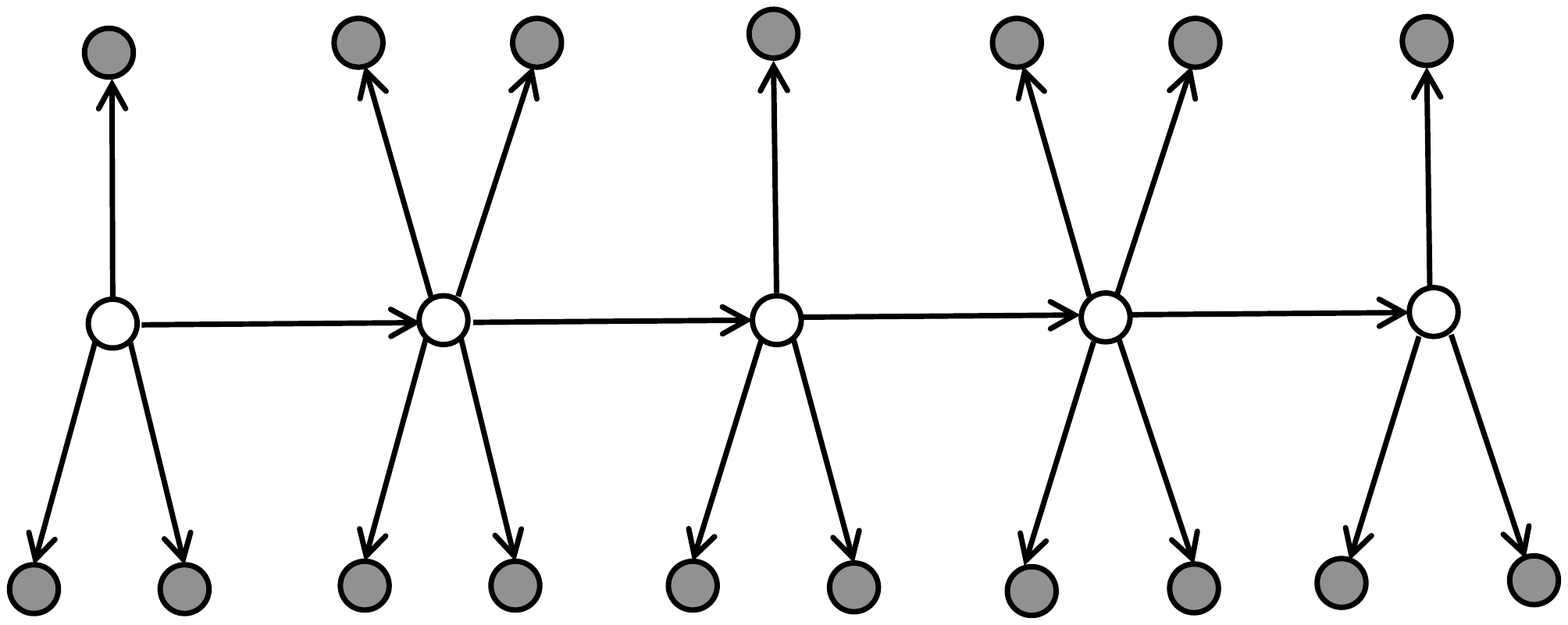}\label{fig:caterpillar}
}
\caption{Illustration of full ternary tree and caterpillar tree. Concrete examples of correlated topic
models than can be learned using low order moments. Words and topics are
respectively shown by shaded and white circles. }
\end{figure}

\subsection{Our techniques}
\label{subsec:techniques}
Our proof techniques rely on ideas and tools developed in dictionary learning,  spectral techniques, and matrix decomposition.
We briefly explain our techniques and their relationships to these areas.

\paragraph{Dictionary learning and $\ell_1$ optimization.}
We  cast the topic models as {\em linear exchangeable multiview models}  in Section~\ref{subsec:model} and demonstrate that the second order (cross) moment between any two words $x_i, x_j$ satisfies
\begin{equation}\label{eqn:summ-secondorder} \Ebb[x_i x_j^\top]= \Ebb[\Ebb[x_i x_j^\top| h ]] = A \Ebb[h h^\top ] A^\top,\quad\forall\,i \neq j , \end{equation} where $A\in \R^{n\times k}$ in the topic-word matrix,   $n$ is the vocabulary size,  $k$ is the number of topics, and $h$ is the topic mixture.
Thus, the problem of learning topic models using second order moments reduces to finding matrix $A$,  given $A \Ebb[h h^\top] A^\top$.

Indeed, further conditions need to be imposed for identifiability of $A$
from $A\Ebb[h h^\top] A^\top$. A natural non-degeneracy constraint is that
the correlation matrix of the hidden topics $\Ebb[h h^\top]$ be  full rank,
so that $\Col(A)=\Col(A\Ebb[h h^\top]A^\top)$, where $\Col(\cdot)$ denotes the column span.
Under the expansion condition in~\eqref{condsummary:expansion}, for generic
parameters, we  establish that  the columns of $A$ are the sparsest vectors
in $\Col(A)$, and are thus identifiable.
To prove this claim, we  leverage ideas from the work of Spielman et. al.~\cite{Spielman-12}, where  the problem of sparsely used dictionaries is considered   under probabilistic assumptions.   In addition, we develop novel techniques to establish non-probabilistic counterpart of the result of~\cite{Spielman-12}. A key ingredient in our proof is  establishing that  submatrices of the topic-word matrix, corresponding to any subset of columns and their neighboring rows, satisfy a certain null-space property under generic parameters and expansion condition in~\eqref{condsummary:expansion}.

The above identifiability result implies recovery of the topic-word matrix
$A$ through exhaustive search for sparse vectors in $\Col(A)$. Instead, we propose an efficient method to recover the columns of $A$ through $\ell_1$ optimization. We prove that $\ell_1$ method recovers the matrix $A$, under the expansion condition in~\eqref{condsummary:expansion}, and some additional conditions on the non-zero entries of $A$.


\paragraph {Spectral techniques for learning latent Bayesian networks.}
When the topic distribution is modeled via a linear Bayesian network,  we exploit additional structure in the observed moments to learn the relationships among the topics, in addition to the topic-word matrix. Specifically,  we assume that the topic variables obey the following linear equations:
\begin{align}
h(j) & = \sum_{\ell \in \PA_j} \lambda_{j\ell} h(\ell) + \eta(j),\quad
\text{for }j \in [k] ,
\label{eqn:hid_eqn}
\end{align}where $\PA_j$ denotes the parents of node $j$ in the directed acyclic graph (DAG) corresponding to the Bayesian network.
Here, we assume that the noise variables $\eta(j)$  are
non-Gaussian (\emph{e.g.},   they have non-zero third moment or excess
kurtosis), and are independent.  We employ the $\ell_1$ optimization framework discussed in the previous paragraph, and in addition,  leverage  the  spectral methods of~\cite{Anima-SVD-12} for learning   using  second and third observed moments.

We first establish that the model in~\eqref{eqn:hid_eqn} reduces to
independent component analysis (ICA),  where the latent variables are
independent components, and this problem can be solved via spectral
approaches (\emph{e.g.},~\cite{Anima-SVD-12}). Specifically, denote $\Lambda=[\lambda_{i,j}]$, where $\lambda_{i,j}$ denotes the dependencies between different hidden topics in~\eqref{eqn:hid_eqn}.
Solving for the hidden topics $h_j$, we have $h =
(I-\Lambda)^{-1} \eta$, where $\eta:= (\eta(1), \dotsc, \eta(k))$ denotes the independent noise variables in~\eqref{eqn:hid_eqn}. Thus, the latent Bayesian network in~\eqref{eqn:hid_eqn} reduces to an ICA model, where
$\eta:= (\eta(1), \dotsc, \eta(k))$ are the independent latent components, and the linear map from hidden  to the  observed variables is given by $A
(I-\Lambda)^{-1}$, where $A$ is the original topic-word matrix. We then apply spectral techniques from~\cite{Anima-SVD-12}, termed as excess correlation analysis (ECA),  to   learn $A(I-\Lambda)^{-1}$ from the
second and third order moments of the observed variables. ECA is based on
two singular value decompositions: the first SVD whitens the data (using
second moment) and the second SVD uses the third moment to find directions which exhibit information that is not captured by the second moment. Finally, in order to recover $A$ from $A(I-\Lambda)^{-1}$, we exploit the expansion property in \eqref{condsummary:expansion}, and extract $A$ as described previously through $\ell_1$ optimization. The high-level idea is depicted in Fig.~\ref{fig:high-level}.


\paragraph{Matrix decomposition into diagonal and low-rank parts for general linear  models.}
Our framework for learning topic models casts them as linear multiview models, where the words represent the multiple views of the  hidden topic mixture $h$, and the conditional expectation of each word given the topic mixture $h$ is a linear map of   $h$.
We   extend our results for learning  general linear models, where such multiple views may not be available. Specifically, we consider \begin{align}
x(i) &= \sum_{j \in \PA_i} a_{ij} h(j) + \veps(i), \quad \text{for } i \in [n]\,, \label{eqn:obs_eqn}
\end{align}
where $\{\veps(i)\}_{i\in[n]}$ are uncorrelated and are independent from the hidden variables $\{h(j)\}_{j\in[k]}$.
In this case, the second order moments
$\Sigma : = \E[xx^\T]$ satisfies
 \begin{eqnarray*}
\Sigma = A \E[hh^\T] A^\T + \E[\veps \veps^\T],
\end{eqnarray*} and has another noise component $\E[\veps \veps^\T]$, when compared to the second-order (cross) moment for topic models in~\eqref{eqn:summ-secondorder}. Note that the rank of $A \E[hh^\T] A^\T$ is $k$  (under non-degeneracy conditions), where $k$ is the number of topics. Thus, when $k$ is sufficiently small compared to $n$, we can view
$\Sigma$ as the sum of a low-rank matrix and a diagonal one.
We prove that under the rank condition that\footnote{It should be noted that other matrix decomposition methods have been considered previously~\cite{Venkat11,HKZ11,Saunderson}.
Using these techniques, we can relax
Condition~\ref{cond:rank} to $k \le n/2$, but only by imposing
stronger incoherence conditions on the low-rank component.} $n \geq 3k$ (and the existence of a partition of three sets of columns of $A$ such that each set has full column rank),
$\E[xx^\T]$ can be decomposed into its low-rank component $A\E[hh^\T] A^\T$
and its diagonal component $\E[\veps \veps^\T]$.
Thus, we employ matrix decomposition techniques to  ``de-noise''   the second order
moment and recover $A\E[hh^\T] A^\T$ from $\Sigma$. From here on, we can apply the techniques described previously to recover $A$ through $\ell_1$ optimization. Thus, we develop novel techniques for learning general latent linear models under expansion conditions.\\

%

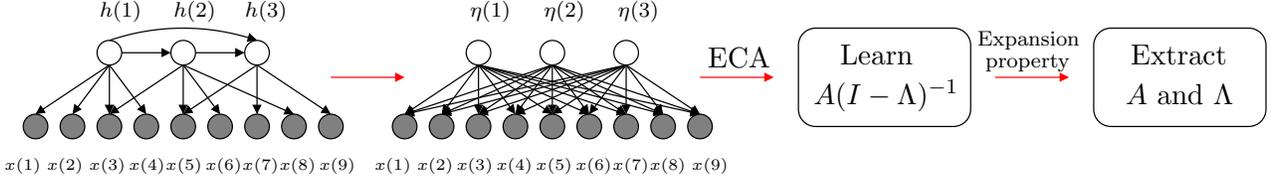
\begin{figure*}[!t]
\includegraphics*[viewport = -3 10 570 60 , width=8.2in]{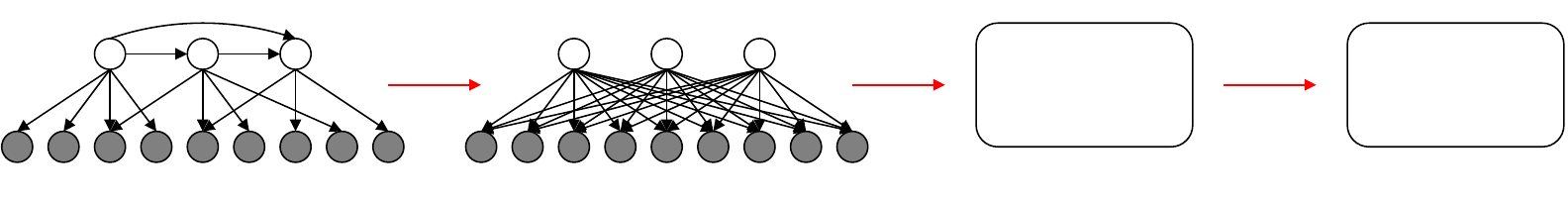}
\put(-596,-10){\tiny{$x(1)$}}
\put(-580,-10){\tiny{$x(2)$}}
\put(-564,-10){\tiny{$x(3)$}}
\put(-549,-10){\tiny{$x(4)$}}
\put(-535,-10){\tiny{$x(5)$}}
\put(-520,-10){\tiny{$x(6)$}}
\put(-506,-10){\tiny{$x(7)$}}
\put(-492,-10){\tiny{$x(8)$}}
\put(-477,-10){\tiny{$x(9)$}}
\put(-560,48){\scriptsize{$h(1)$}}
\put(-532,48){\scriptsize{$h(2)$}}
\put(-505,48){\scriptsize{$h(3)$}}
\put(-456,-10){\tiny{$x(1)$}}
\put(-440,-10){\tiny{$x(2)$}}
\put(-425,-10){\tiny{$x(3)$}}
\put(-410,-10){\tiny{$x(4)$}}
\put(-395,-10){\tiny{$x(5)$}}
\put(-380,-10){\tiny{$x(6)$}}
\put(-366,-10){\tiny{$x(7)$}}
\put(-352,-10){\tiny{$x(8)$}}
\put(-336,-10){\tiny{$x(9)$}}
\put(-420,48){\scriptsize{$\eta(1)$}}
\put(-392,48){\scriptsize{$\eta(2)$}}
\put(-364,48){\scriptsize{$\eta(3)$}}
\put(-330,29){ECA}
\put(-280,30){Learn}
\put(-290,15){$A(I-\Lambda)^{-1}$}
\put(-228,37){\scriptsize{Expansion}}
\put(-225,29){\scriptsize{property}}
\put(-170,30){Extract}
\put(-172,15){$A$ and $\Lambda$}
\begin{center}
\end{center}
\caption{The high-level idea of the technique used for learning latent Bayesian networks. In the leftmost graph (original DAG) the hidden nodes depend on each other through the matrix $\Lambda$ and the observed variables depend on the hidden nodes through the coefficient matrix $A$. We consider an equivalent DAG with new independent latent variables $\eta_j$ (these are in fact the noise terms at the hidden nodes in the previous model). Here, the observed variables depend on the hidden ones through the matrix $A(I-\Lambda)^{-1}$. Applying ECA method, we learn this matrix from the (second and third order) observed moments. Finally, using the expansion property of the connectivity structure between the hidden part and the observed part, we extract $A$ and $\Lambda$ from $A(I-\Lambda)^{-1}$.}\label{fig:high-level}
\end{figure*}

%
%
%
%

Our presentation focuses on using exact (population) observed moments to
emphasize the correctness of the methodology.
However, ``plug-in'' moment estimates can be used with sampled data.
To partially address the statistical efficiency of our method,  note that
higher-order empirical moments generally have higher variance than lower-order empirical moments, and therefore are more difficult to reliably estimate. Our techniques only involve low-order moments (up to third order). A precise analysis of sample complexity involves
standard techniques for dealing with sums of i.i.d.\ random matrices and
tensors as in~\cite{Anima-SVD-12} and is left for future study.
See Section~\ref{sec:simulation} for the performance of our proposed algorithms under finite number of samples.

%% file: related.tex
\subsection{Related work}
\label{subset:related work}

Probabilistic topic models have received widespread attention in recent
years; see~\cite{blei2012probabilistic} for an overview. However, till
recently, most learning approaches do not have provable guarantees, and in
practice Gibbs sampling or variational Bayes methods are used. Below, we
provide an overview of learning approaches with theoretical guarantees.

\paragraph{Learning topic models through moment-based approaches.} A  series of recent works aim to learn topic models using low order moments (second and third) under parametric assumptions on the topic distribution, e.g. single-topic model~\cite{AHK12} (each document consists of a single topic), latent Dirichlet allocation (LDA)~\cite{Anima-SVD-12}, independent components analysis (ICA)~\cite{ICAbook} (the different components of $h$, i.e., $h_i$ are independent), and so on; see~\cite{AGHKT} for an overview. A general framework based on tensor decomposition is given in~\cite{AGHKT} for a wide range of latent variable models, including LDA and single topic models, Gaussian mixtures, hidden Markov models (HMM), and so on. These approaches do not impose any constraints  on the topic-word matrix $A$ (other than  non-degeneracy). In contrast, in this paper, we impose constraints on $A$, and allow for any general topic distribution. Furthermore, we specialize the results to parametric settings where the topic distribution is a Bayesian network, and for this sub-class,  we use ideas from the method of moments (in particular, the excess correlation method (ECA) of~\cite{Anima-SVD-12})  in conjunction with ideas from sparse dictionary learning.

\paragraph{Learning topic models through non-negative matrix factorization.} Another series of recent works by Arora et. al.~\cite{arora2012learning,AroraICML} employ a similar philosophy as this paper: they  allow for general topic distributions, while constraining the topic-word matrix $A$. They employ approaches based on  non-negative matrix factorization (NMF), and exploit the fact that $A$ is non-negative (recall that $A$ corresponds to conditional distributions). The approach and the assumptions are quite different from this work. They establish guaranteed learning under the assumption that every topic has an {\em anchor} word, i.e. the word is uniquely generated from the topic, and does not occur under any other topic (with reasonable probability).  Note that the presence of anchor words implies expansion constraint: $|\Neigh(S)|\geq |S|$ for all subsets $S$ of topics, where $\Neigh(S)$ is the set of neighboring words for topics in $S$. In contrast, our requirement for guaranteed learning is $|\Neigh(S)|\geq |S|+d_{\max}$, where $d_{\max}$ is the maximum degree of any topic.  Thus our requirement  is comparable to $|\Neigh(S)|\geq |S|$,  when $d_{\max}$ is small, and our   approach does not require presence of anchor word. Additionally,  our approach does not assume that the topic-word matrix $A$ is positive, which makes it   applicable for more general linear models, e.g. when the variables are not discrete and matrix $A$  corresponds to a general mixing matrix (note that for discrete variables, $A$ corresponds to conditional distribution and is thus non-negative).

\paragraph{Dictionary learning.}  As discussed in Section~\ref{subsec:techniques}, we use some of the the ideas developed in the context of sparsely used dictionary learning problem.
The problem setup there is that one is given a matrix $X$ and is asked to find a pair of matrices $A$ and $M$ so that $\|X-AM\|$ is small and also $M$ is sparse. Here, $A$ is considered as the dictionary being used.  Spielman et. al~\cite{Spielman-12} study this problem assuming that $A$ is a full rank square matrix and the observation $X$ is noiseless, i.e., $X = A M$. In this scenario, the problem can be viewed as learning a matrix $X$ from its row space knowing that $X$ enjoys some sparsity structure. Stating the problem this way clearly describes the relation to our work, as we also need to recover the topic-word matrix $A$ from its second-order moments $A\E[hh^\T]A^\T$, as explained in Section~\ref{subsec:techniques}.

The results of~\cite{Spielman-12} are obtained assuming  that  the entries of $M$ are drawn i.i.d. from a Bernoulli-Gaussian distribution. The idea is then to seek the rows of $X$
sequentially, by looking for the sparse vectors in $\Row(Y)$.
Leveraging similar ideas, we obtain non-probabilistic counterpart of the results, i.e., without assuming any parametric distribution on the topic-word matrix. These conditions turn out to be intuitive  expansion conditions on the support of the topic-word matrix, assuming generic parameters. Our technical arguments to arrive at these results are  different than the ones employed in~\cite{Spielman-12}, since we do not assume any parametric distribution, and its application to learning topic models is novel. Moreover,  in fact, it can be shown that the considered probabilistic models considered our~\cite{Spielman-12}, satisfy the expansion  property~\eqref{condsummary:expansion} almost surely, and are thus, special cases under our framework. Variants of the sparse dictionary learning problem of~\cite{Spielman-12} have also been proposed~\cite{zibulevsky2001blind,gottlieb2010matrix}. For a detailed discussion on other works dealing with dictionary learning, refer to~\cite{Spielman-12}.

\paragraph{Linear structural equations.}
In general, structural equation modeling (SEM) is defined by a collection of equations $z_i = f_i(z_{\PA_i},\veps_i)$, where $z_i$'s are the variables associated to the nodes. Recently, there has been some progress on the identifiability problem of SEMs in the fully observed linear models~\cite{Shimizu-06,Hoyer-09, Peters-11,Peters-12}. More specifically, it has been shown that for linear functions $f_i$ and non-Gaussian
noise, the underlying graph $\G$ is identifiable~\cite{Shimizu-06}. Moreover, if one restricts
the functions to be additive in the noise term and excludes the linear Gaussian case
(as well as a few other pathological function-noise combinations), the graph structure $\G$ is
identifiable~\cite{Hoyer-09, Peters-11}. Peters et. al.~\cite{Peters-12} consider Gaussian SEMs
with linear functions, and the normally distributed noise variables with the same
variances and show that the graph structure $\G$ and the functions are identifiable. However, none of these works deal with latent variables, or address the issue of efficiently learning the models. In contrast, our work here can be viewed as a contribution to the problem of identifiability and learning of linear SEMs with latent variables.

\paragraph{Learning Bayesian networks and undirected graphical models. }
The problem of identifiability and learning graphical models from distributions has been the object of intensive investigation in the past years and has been studied in different research communities. This problem has proved important in a vast number of applications, such as computational biology~\cite{Durbin-98, Roch-12}, economics~\cite{Awokuse-03, Haavelmo-43,Zellner-71,Bollen-89}, sociology~\cite{Bagozzi-80,Kohn-82,Wheaton-78,Bollen-89}, and computer vision~\cite{Koller:UAI97,Choi-10}.  The learning task has two main ingredients: structure learning and parameter estimation.


Structure estimation of probabilistic graphical models has been extensively
studied in the recent years. It is well known that maximum likelihood
estimation in fully observed tree models is tractable~\cite{Chow&Liu:68IT}.
However, for general models, maximum likelihood structure learning is
NP-hard even when there are no hidden variables. The main approaches for
structure estimation are score-based methods, local tests and convex
relaxation methods. Score-based methods such as~\cite{Chickering-03} find
the graph structure by optimizing a score (\emph{e.g.}, Bayesian Independence Criterion) in a greedy manner. Local test approaches attempt to build the graph based on local statistical tests on the samples, both for directed and undirected graphical models~\cite{Spirtes-00,Abbeel2006learning,Bresler&etal:Rand,AnandkumarTanWillsky:Ising11,Jalali:greedy,hauser2011characterization}.
Convex relaxation approaches have also been considered for structure
estimation (\emph{e.g.},~\cite{Mei06,Ravikumar&etal:08Stat}).

In the presence of latent variables, structure learning becomes more challenging. A popular class of latent variable models are latent trees, for which efficient algorithms have been developed~\cite{erdos99,daskalakis06,Choi&etal:10JMLR,Anima-spectral-11}. Recently, approaches have been proposed for learning (undirected) latent graphical models with long cycles in certain parameter regimes~\cite{Anandkumar:girth12}.
In~\cite{Chandrasekaran:10latent}, latent Gaussian graphical models are estimated using convex relaxation approaches.
The authors in~\cite{Silva-06} study linear latent DAG models and propose methods to (1) find clusters of observed nodes that are separated by a single latent common cause; and (2) find features of the Markov Equivalence class of causal models for the latent variables. Their model allows for undirected edges between the observed nodes.
In~\cite{ali2005towards}, equivalence class of DAG models is characterized when there are latent variables.  However, the focus is on constructing an equivalence class of DAG models, given a member of the class. In contrast, we focus on developing efficient learning methods for latent Bayesian networks based on spectral techniques in conjunction with $\ell_1$ optimization.


%% file: model-results.tex
\section{Model and sufficient conditions for identifiability}

\paragraph{Notation.} We write $\|v\|_p$ for the standard $\ell^p$ norm of a vector $v$. Specifically, $\|v\|_0$
denotes the number of non-zero entries in $v$. Also, $\|M\|_p$ refers to the induced operator norm on a matrix $M$.
For a matrix $M$ and set of indices $I,J$, we let
$M_I$ denote the submatrix containing just the rows in $I$ and $M_{I,J}$ denote the submatrix formed by the rows in $I$ and columns in $J$.
For a vector $v$, $\supp(v)$ represents the positions of non-zero entries of $v$.
We use $e_i$ to refer to the $i$-th standard basis element, \emph{e.g.}, $e_1 = (1, 0, \dotsc, 0)$.
For a matrix $M$ we let $\Row(M)$ (similarly $\Col(M)$) denote the span of its rows (columns). For a set $S$, $|S|$ is its cardinality.
We use the notation $[n]$ to denote the set $\{1, \dotsc, n\}$.
For a vector $v$, $\diag(v)$ is  a diagonal matrix with the elements of $v$ on the diagonal. For a matrix $M$,
$\diag(M)$ is a diagonal matrix with the same diagonal as $M$. Throughout $\otimes$ denotes the tensor product.

\subsection{Overview of topic models}

\begin{figure}
\begin{center}
\begin{tikzpicture}
  [
    scale=1.0,
    observed/.style={circle,minimum size=0.7cm,inner sep=0mm,draw=black,fill=black!20},
    hidden/.style={circle,minimum size=0.7cm,inner sep=0mm,draw=black},
  ]
  \node [hidden,name=h] at ($(0,0)$) {$h$};
  \node [observed,name=x1] at ($(-1.5,-1)$) {$x_1$};
  \node [observed,name=x2] at ($(-0.5,-1)$) {$x_2$};
  \node at ($(0.5,-1)$) {$\dotsb$};
  \node [observed,name=xl] at ($(1.5,-1)$) {$x_\ell$};
  \draw [->] (h) to (x1);
  \draw [->] (h) to (x2);
  \draw [->] (h) to (xl);
\end{tikzpicture}\end{center}
\caption{Exchangeable topic model with topic mixture $h$ and $x_i $ represents $i$-th word  in the document.}\label{fig:exchangeable}\end{figure}
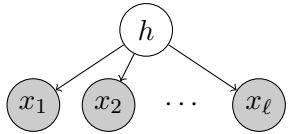

Consider the  {\em bag-of-words} model for documents in which the sequence of observed
words $x_1, x_2, \ldots,x_\ell$ in the document are \emph{exchangeable}, i.e., the joint probability distribution  is invariant to
permutation of the indices.
The well-known De Finetti's theorem~\cite{austin2008exchangeable} implies
that such exchangeable models can be viewed as mixture models in which
there is a latent   variable $h$ such that $x_1, x_2, \dotsc, x_\ell$ are
conditionally i.i.d.~given $h$  and the conditional distributions are
identical at all the nodes. See Fig.\ref{fig:exchangeable} for an illustration.

In the context of document modeling, the latent variable $h$ can be interpreted as a distribution over the topics occurring in a  document. If the total number of  topics is $k$, then $h$ can be viewed as a distribution over the simplex $\Delta^{k-1}$. The word generation process is thus a hierarchical process:  for each document, a realization of $h$ is drawn and it represents the proportion of topics in the documents, and for each word, first a topic is drawn from the topic mixture, and then the word is drawn given the topic.

Let $A=[a_{ij}]\in \R^{n \times k}$ denote the topic-word matrix, where $a_{i,j}$ denotes the conditional probability of word $i$ occurring given that the topic $j$ was drawn.
It is convenient to represent the   words in the document by
$n$-dimensional random \emph{vectors} $x_1, x_2, \dotsc, x_\ell \in \R^n$.
Specifically, we set
\[ x_t = e_i \quad \text{if and only if} \quad
\text{the $t$-th word in the document is $i$} , \quad t \in [\ell] , \]
where $e_1,e_2,\ldots e_n$ is the standard coordinate basis for $\R^n$.

The above encoding allows for a convenient representation of topic models as  linear models:
\[ \Ebb[x_i | h ] = Ah, \quad\forall\, i \in [l],\] and moreover the second order cross-moments (between two different words) have a simple form:
\begin{equation}\label{eqn:secondorder} \Ebb[x_i x_j^\top]= \Ebb[\Ebb[x_i x_j^\top| h ]] = A \Ebb[h h^\top ] A^\top,\quad\forall\,i \neq j .\end{equation}
Thus, the above representation allows us to view topic models as linear models. Moreover, it allows us to incorporate other linear models, i.e. when $x_i$ are not basis vectors. For instance, the independent components model is a popular framework, and can be viewed as a set of linear structural equations with latent variables.  See Section~\ref{sec:singleview} for a detailed discussion.

Thus, the learning task using second-order (exact) moments in \eqref{eqn:secondorder} reduces to    recovering $A$ from
$A \Ebb[h h^\top ] A^\top$, or equivalently $A \Ebb[hh^\top]^{1/2}$.

\subsection{Sufficient conditions for identifiability}
\label{subsec:model}

We first start with some natural non-degeneracy conditions.

\begin{condition}[Non-degeneracy]\label{cond:nondegeneracy}
The  topic-word matrix $A:=[a_{i,j}]\in \R^{n \times k}$ has full column
rank and the hidden variables are linearly independent, \emph{i.e.}, with
probability one, if $\sum_{i\in [k]} \alpha_i h(i) = 0$, then $\alpha_i = 0$, for all $i \in [k]$.
\end{condition}
We note that without such non-degeneracy assumptions,
there is no hope of distinguishing different hidden nodes.

%


We now describe sufficient conditions under which the topic model becomes identifiable using second order observed moments.
Given word observations $x_1, x_2,\ldots$, note that we can only hope to
identify the columns of topic-word matrix $A$ up to permutation because the
model is unchanged if one permutes the hidden variable $h$ and the columns
of $A$ correspondingly. Moreover, the scale of each column of $A$ is also
not identifiable. To see this, observe that Eq.~\eqref{eqn:secondorder} is
unaltered if we both rescale all the coefficients $\{a_{ij}\}_{i \in [n]}$
and appropriately rescale the variable $h(j)$. Without further assumptions, we can only hope to recover a certain canonical form of $A$, defined as follows:

\begin{definition}
We say $A$ is in a \emph{canonical form} if all of its columns have unit norm. In particular, the transformation $A \leftarrow A \diag(\|A_{[n],1}\|^{-1}, \|A_{[n],2}\|^{-1}, \dotsc, \|A_{[n],k}\|^{-1})$ and the corresponding rescaling of $h$ place $A$ in canonical form and the distribution over $x_i$, $i\in [n]$, is unchanged.
\end{definition}
Furthermore, observe that the canonical $A$ is only specified up to sign of
each column since any sign change of column $i$
does not alter its norm.

Thus, under the above non-degeneracy and scaling conditions, the task of recovering $A$ from  second-order (exact) moments in \eqref{eqn:secondorder} reduces to    recovering $A$ from Col$(A)$. Recall that our criterion for identifiability is that the sparsest vectors in the Col$(A)$ correspond to the columns of $A$. We now provide sufficient conditions for this to occur, in terms of structural conditions on the support of $A$, and parameter conditions on the non-zero entries of $A$.

For structural conditions on the topic-word matrix $A$, we proceed by defining the \emph{expansion property} of a graph which plays a key role in establishing our identifiability results.

\begin{condition}[Graph expansion] \label{cond:expansion}
Let $\H(\vhid,\vobs)$ denote the bipartite graph formed by the support of $A$: $\H(i,j)=1$ when $a_{i,j}\neq 0$, and $0$ otherwise, and $\vhid:=[k]$, $\vobs:=[n]$. We assume that the $\H$ satisfies the following expansion property:
\begin{equation}\label{eqn:mainexpansion} |\Neigh(S)| \ge |S| + d_{\max}, \quad \forall\,S\subset [k], |S|\geq 2,\end{equation}where $\Neigh(S): = \{i \in \V_2: (j,i) \in \mathcal{E} \text{ for some }j \in S\}$ is the set of the neighbors of $S$ and $d_{\max}$ is the maximum degree of nodes in $\vhid$.
\end{condition}

Note that the condition $|\Neigh(S)| \geq |S|$, for all subsets of hidden nodes $S\subset [k]$, is {\em necessary} for the matrix $A$ to be full column rank. We observe that the above sufficient condition in \eqref{eqn:mainexpansion} has an additional degree term $d_{\max}$, and  is thus close to the necessary condition when $d_{\max}$ is small. Moreover, the above condition in \eqref{eqn:mainexpansion}  is only a weak additive expansion, in contrast to multiplicative expansion, which is typically required for  various properties to hold, e.g.~\cite{berinde2008combining}.

The last condition is a generic assumption on the entries of matrix $A$. We
first define the \emph{parameter genericity property} for a matrix.

\begin{condition}[Parameter genericity] \label{cond:parameter}
We assume that the topic-word matrix $A$ has the following parameter genericity property:
for any $v \in \R^k$ with $\|v\|_0 \ge 2$, the following holds true.
\begin{eqnarray}
\|Av\|_0 > |\Neigh_A(\supp(v))| - |\supp(v)|,
\end{eqnarray}
where for a set $S \subseteq [k]$, $\Neigh_A(S) := \{i\in [n]: A_{ij} \neq 0 \text{ for some j }\in S\}$.\end{condition}

This is a mild generic condition. More specifically if the entries of any arbitrary fixed matrix $M$ are perturbed independently, then it satisfies the above generic property with probability one.

\begin{remark}
\label{rem:genericity}
Fix any matrix $M \in \R^{n \times k}$.
Let $Z \in \R^{n \times k}$ be a random matrix such that $\{ Z_{ij}:
M_{ij} \neq 0 \}$ are independent random variables, and $Z_{ij}
\equiv 0$ whenever $M_{ij} = 0$.
Assume each variable is drawn from a distribution with uncountable support.
Then
\begin{eqnarray}
\P(\text{$M + Z$ does not satisfy Condition~\ref{cond:parameter}}) = 0 .
\end{eqnarray}
\end{remark}

Remark~\ref{rem:genericity} is proved in Appendix~\ref{app:genericity}.

%
%
%
%

%% file: main-results.tex
\section{Identifiability result and Algorithm}\label{sec:main-results}

In this section, we state our identifiability results and algorithms for learning the topic models under expansion conditions.

\begin{theorem}[Identifiability of the Topic-Word Matrix]\label{thm:core}
Let $\Pairs: = \E[x_1 \otimes x_2]$ be the pairwise correlation of the words. For the model described in Section~\ref{subsec:model}
(Conditions~\ref{cond:nondegeneracy}, \ref{cond:expansion},
\ref{cond:parameter}), all columns of $A$ are identifiable from $\Pairs$.
\end{theorem}


Theorem~\ref{thm:core} is proved in Section~\ref{sec:core-proof}. As shown in the proof, columns of $A$ are in fact the sparsest vectors in the space $\Col(A \E[hh^\T] A^\T)$. This result already implies identifiability of $A$ via an exhaustive search, which is an interesting result in its own right. The following theorem provides some conditions under which the columns of $A$ can be identified by solving a set of convex optimization problems. Before stating the theorem, we need to establish some notations.

For $i\in[n]$, we define $\Neigh_i := \{j\in [k]: A_{ij} \neq 0 \}$ and $\Neigh^2_i := \{l\in [n]: A_{lj} \neq 0 \text{ for some } j\in \Neigh_i\}$. Similarly, for $j\in [k]$, define  $\Neigh_j := \{i\in [n]: A_{ij} \neq 0 \}$ and $\Neigh^2_j := \{l\in [k]: A_{il} \neq 0 \text{ for some } i\in \Neigh_j\}$. Thus, for a node $i$ (either a topic or a word), $\Neigh_i$ is the set of its neighbors and $\Neigh^2_i$ represents the set of nodes with distance exactly two from $i$. Therefore, if $i$ is a word node, $\Neigh^2_i$ is the set of its siblings and if $i$ is a topic word, $\Neigh^2_i$ is the set of topics with a common child.
 We further use superscript $c$ to denote the set complement.

\begin{theorem}[Recovery of the Topic-Word Matrix through $\ell_1$-minimzation]\label{thm:alg1}
Suppose that in each row of $A$, there is a gap between the maximum and the second maximum absolute values. For $i\in [n]$, let $\pi_i$ be a permutation such that $|a_{i,\pi_i(1)}| \ge |a_{i,\pi_i(2)}| \ge \dotsb \ge |a_{i,\pi_i(k)}|$, and $|a_{i,\pi_i(2)}| / |a_{i,\pi_i(1)}| \le 1 - \gamma_i$, for some $\gamma_i > 0$. Further suppose that $[k] \subseteq \{\pi_1(1), \dotsc, \pi_n(1)\}$. In words, each column contains at least one entry that has the maximum absolute value in its row. If the following conditions hold true for $i \in [n]$, then $\AlgI$ returns the columns of $A$ in canonical form.

\begin{itemize}
\item [(i)]
$\|A_{(\Neigh^2_{i})^c, (\Neigh_i)^c}\, v\|_1 >
\|A_{\Neigh^2_{i},(\Neigh_i)^c}\, v\|_1$
for all non-zero vectors $v \in \R^{|(\Neigh_i)^c|}$.

\item[(ii)]
$\|A_{(\Neigh_j)^c, \Neigh_i \backslash j} \, v\|_1 > \|A_{\Neigh_j, \Neigh_i \backslash j}\, v\|_1 + (1-\gamma) \|A_{\Neigh_j,j}\|_1 \|v\|_1$
for all $j \in \Neigh_i$ and
all non-zero vectors $v \in \R^{|\Neigh_i|-1}$.
\end{itemize}
\end{theorem}

\begin{algorithm}
\caption*{$\AlgI$: Learning the topic-word matrix form pairwise correlations ($\Pairs$).}
\begin{algorithmic}[1]

\REQUIRE Pairwise correlation of the words ($\Pairs$).

\ENSURE Columns of $A$ up to permutation.



\FOR{each $i\in [n]$}

\STATE Solve the optimization problem\footnote{In this paper, when $A= B B^\top$, we use the notation $B=A^{1/2}$ which differs from the standard definition of matrix square root.}
\[ \min_{w} \,\, \|\Pairs^{1/2} w\|_1 \quad \quad \text{subject to } (e_i^\T
\Pairs^{1/2}) w = 1 . \]

\STATE Set $s_i = \Pairs^{1/2} w$, and let $\mathcal{S} = \{s_1,\dotsc, s_n\}$.

\ENDFOR

\FOR{each $j = 1,\dotsc,k$}

\REPEAT

\STATE Let $v_j$ be an arbitrary element in $\mathcal{S}$.

\STATE Set $\mathcal{S} = \mathcal{S} \backslash\{v_j\}$.

\UNTIL{$\rank([v_1 | \dotsb | v_j]) = j$}

\ENDFOR



\RETURN $\widehat{A} = \Big[\frac{v_1}{\|v_1\|}\Big|\dotsb \Big|\frac{v_k}{\|v_k\|}\Big]$.

\end{algorithmic}
\end{algorithm}

Theorem~\ref{thm:alg1} is proved in Section~\ref{subsec:alg1-proof}. $\AlgI$ is essentially the {\sf ER-SpUD} presented in~\cite{Spielman-12} for exact recovery of sparsely-used dictionaries, but the technical result and application in Theorem~\ref{thm:alg1} are novel.

$\AlgI$ involves solving $n$ optimization problems and as the number of words becomes large, this requires a fast
method to solve $\ell_1$ minimization. Traditionally, the $\ell_1$ minimization can be formulated as a linear programming
(LP) problem. In particular, each of the $\ell_1$ minimizations in $\AlgI$ can be written as an LP
with $2(n-1)$ inequality constraints and one equality constraint. However, the computational complexity of such a general-purpose formulation is often too high for large scale
applications. Alternatively, one can use approximate methods which are significantly faster. There are several relevant algorithms with this theme, such as gradient projection~\cite{Figueiredo-gradient,Kim-gradient}, iterative shrinkage-thresholding~\cite{Daubechies-IST}, and proximal gradient (Nestrov's method)~\cite{Nestrov-1,Nestrov-2}.

%% file: BN.tex
\section{Bayesian networks for modeling topic distributions}\label{subsec:BN}

According to Theorem~\ref{thm:core}, we can learn the topic-word matrix
$A$ without any assumption on the dependence relationships
among the  hidden topics.
(We only need the non-degeneracy assumption discussed in
Condition~\ref{cond:nondegeneracy} which requires the hidden variables to be
linearly independent with probability one.)


Bayesian networks provide a natural framework for modeling topic dependencies, and we employ them here for modeling topic distributions.    For these families, we prove identifiability and learning of the entire model, including the topic relationships and the topic-word matrix.

Bayesian networks, also known as belief networks, incorporate a set of causal and conditional independence  through  directed acyclic graphs (DAG)~\cite{Pearl:book}. They have  widespread applicability in artificial intelligence~\cite{Koller:SRL07,Boutilier:UAI96, Koller:UAI97, Choi-10}, in the social sciences~\cite{Bagozzi-80,Kohn-82,Wheaton-78,Bollen-89, Pearl-00, Pearl-soc-98}, and as structural equation models in economics~\cite{Awokuse-03, Haavelmo-43,Zellner-71,Bollen-89, Pearl-00, Spirtes-05}.

We define a \emph{DAG model} as a pair $(\G, \P_{\theta})$, where
$\P_{\theta}$ is a joint probability distribution, parameterized by
$\theta$, on $k$ variables $h := (h(1),\dotsc,h(k))$ that is Markov with
respect to a DAG $\G = (\H,\edges)$ with $\H = \{1,\dotsc,k\}$~\cite{Lauritzen-96}.
More specifically, the joint probability $\P_{\theta}(h)$ factors as
\begin{eqnarray}
\label{eqn:DAG-model}
\P_{\theta}(h) = \prod_{i=1}^k \P_{\theta}(h(i)| h_{\PA_i}),
\end{eqnarray}
where $\PA_i:= \{ j \in \V : (j,i) \in \edges \}$ denotes the set of parents of node $i$ in $\G$.

\aj{The learning task involving DAG models can be described as: \emph{Given i.i.d.\ samples generated from the joint
distribution $\P_{\theta}$ over $x_S$ for some $S \subseteq \V$, recover
(some part of) the graph structure $\G$ and estimate the model parameter
$\theta$.}}

%


\aj{\subsubsection{DAGs with effective depth one}

We consider the subclass of DAGs are those with effective depth one.
\begin{definition}
The \emph{effective depth} of a DAG model with hidden nodes is the maximum graph distance between a hidden node and its closest observed node.
\end{definition}

In particular, in a DAG with effective depth one every hidden node has at least one observed neighbor.
Assume further that the hidden variables obey the linear model}

We consider a subclass of DAG models for the topics in which the topics obey the linear relations
\begin{eqnarray}\label{eqn:hid-2}
h(j) = \sum_{\ell\in \PA_j} \lambda_{j\ell} h(\ell) + \eta(j)\,, \quad \text{for }
j \in [k] \,,
\end{eqnarray}
where $\eta(j)$ represents the noise variable at topic $j$. We further assume that the noise variables $\eta(j)$ are independent.


Let $\Lambda \in \R^{k \times k}$ be the matrix with $\lambda_{ij}$ at the $(i,j)$ entry if $j \in \PA_i$ and zero everywhere else.
Without loss of generality, we
assume that hidden (topic) variables $h(j)$, the observed (word) variables $x(i)$ and the
noise terms $\veps(i), \eta(j)$ are all zero mean. We also denote the
variances of $\veps(i)$ and $\eta(j)$ by $\sigma^2_{\veps(i)}$ and
$\sigma^2_{\eta(j)}$, respectively. Let $\mu_{\veps(i)}$ and $\mu_{\eta(j)}$
respectively denote the third moment of $\veps(i)$ and $\eta(j)$, \emph{i.e.}, $\mu_{\veps(i)}:= \E[\veps(i)^3]$ and $\mu_{\eta(j)}:= \E[\eta(j)^3]$. Define the skewness of $\eta(j)$ as:
\begin{eqnarray}
\gamma_{\eta(j)}:= \frac{\mu_{\eta(j)}}{\sigma^3_{\eta(j)}}\,.
\end{eqnarray}

Finally, define the following moments of the observed variables:
\begin{equation} \label{eq:moments}
\begin{aligned}
\Pairs &:= \E[x_1 \otimes x_2],\\
\Triples &:= \E[x_1 \otimes x_2 \otimes x_3]\,.
\end{aligned}
\end{equation}
It is convenient to consider the projection of $\Triples$ to a matrix as follows:
\begin{eqnarray*}
\Triples(\zeta):= \E[ x_1\otimes x_2 \, \langle \zeta, x_3 \rangle]\,,
\end{eqnarray*}
where $\langle \cdot,\cdot \rangle$ denotes the standard inner product.

\begin{theorem}\label{thm:eff-depth}
Consider a DAG model which satisfies the model conditions described in Section~\ref{subsec:model} and the hidden variables are related through linear equations~\eqref{eqn:hid-2}. If the noise variables $\eta(j)$ are independent and have non-zero skewness for $j \in [k]$, then the DAG model is identifiable from $\Pairs$ and $\Triples(\zeta)$, for an appropriate choice of $\zeta$. Furthermore, under the assumptions of Theorem~\ref{thm:alg1}, $\AlgII$ returns matrices $A$ and $\Lambda$ up to a permutation of hidden nodes.
\end{theorem}

Theorem~\ref{thm:eff-depth} is proven in Section~\ref{subsec:eff-depth-proof}.

Notice that the only limitations on the noise variables $\eta(j)$ are that they are independent\footnote{We only require
pairwise and triple-wise independence.}, and have non-zero skewness. Some common examples of non-zero skewness
distributions are exponential, chi-squared and Poisson. Note that different topics may have different noise distributions.


\begin{remark}[Special Cases]A special case of the above result is when the DAG is empty, i.e. $\Lambda=0$, and the topics $h(1), \ldots, h(k)$ are independent. This is popularly known as the independent components model (ICA), and similar spectral techniques have been proposed before for learning ICA~\cite{ICAbook}. Similarly, the ECA approach proposed above is also applicable for learning latent Dirichlet allocation (LDA), using suitably adjusted second and third order moments~\cite{Anima-SVD-12}. Note that for these special cases, we do not need to impose any constraints on the topic-word matrix $A$ (other than non-degeneracy), since we can directly learn $A$ and the topic distribution through ECA.
\end{remark}

Another immediate application of the technique used in the proof of Theorem~\ref{thm:eff-depth} is in learning fully-observed linear Bayesian networks.
\begin{remark}[Learning fully-observed BN's]
 Consider an \emph{arbitrary} fully-observed linear DAG:
\begin{eqnarray}
x(i) = \sum_{j \in \PA_i} \lambda_{ij} x(j) + \eta(i), \quad \text{ for }i\in [n],
\end{eqnarray}
 and suppose that the noise variables $\eta(i)$ have non-zero skewness. Then, applying the same argument as in the proof of Theorem~\ref{thm:eff-depth}, we can learn the matrix $(I-\Lambda)^{-1}$ (and hence $\Lambda$) from the second and third order moments (We have $A=I$ here).
\end{remark}

\begin{algorithm}
\caption*{$\AlgII$: Learning topic models with correlated topics.}
\begin{algorithmic}[1]
\renewcommand{\algorithmicwhile}{}
\renewcommand{\algorithmicdo}{}


\REQUIRE Observable moments $\Pairs$ and $\Triples$
as defined in Eq.~\eqref{eq:moments}.

\ENSURE Columns of $A$, matrix $\Lambda$ (in a topological ordering).





\WHILE{\textbf{Part 1: ECA.}}

\STATE Find a matrix $U \in \R^{n \times k}$ such that $\Col(U)$ =
$\Col(\Pairs)$.

\STATE Find $V \in \R^{k \times k}$ such that $V^\T(U^\T \Pairs U) V
= I_{k \times k}$.
Set $W = UV$.


\STATE Let $\theta\in \R^k$ be chosen uniformly at random over the unit sphere.

\STATE Let $\Omega$ be the set of (left) singular vectors, with unique singular values, of $W^\T \Triples(W\theta) W$.

\STATE Let $S \in \R^{n\times k}$ be a matrix with columns $\{(W^+)^\T
\omega: \omega\in \Omega\}$, where $W^+ = (W^\T W)^{-1} W^\T$.

\ENDWHILE

\WHILE{\textbf{Part 2: Finding $A$ and $\Lambda$.}}

\STATE Let $\widehat{A} = \AlgI(\Pairs)$.

\STATE Let $\widehat{B}$ be a left inverse of $\widehat{A}$. Let $C =
\widehat{B} S$.

\STATE Reorder the rows and columns of $C$ to make it lower triangular.
Call it $\tilde{C}$.

\ENDWHILE

\RETURN Columns of $\widehat{A}$ and $\widehat{\Lambda} = I -
\diag(\tilde{C})\, \tilde{C}^{-1}$.

\end{algorithmic}
\end{algorithm}

For sake of simplicity, $\AlgII$ is presented using the ECA method, which
uses a single random direction $\theta$ and obtaining singular vectors of
$W^\T \Triples(W\theta) W$.
A more robust alternative to this, as described in~\cite{AGHKT}, is to use
the following power iteration to obtain the singular vectors
$\{v_1,\dotsc,v_k\}$; we use this variant in the simulations described in
Section~\ref{sec:simulation}.

%
{\small
\begin{center}{
\fbox {
    \parbox[c][][c]{3.3in}{
    $\{v_1,\dotsc,v_k\} \leftarrow$ random orthonormal basis for $\R^k$.
    Repeat:
    \begin{enumerate}
    \item For $i=1, 2, \dotsc, k:$
    \begin{itemize}
    \item $v_i \leftarrow W^\T \Triples(Wv_i) Wv_i$.
    \end{itemize}
    \item Orthonormalize $\{v_1,\dotsc,v_k\}$.
    \end{enumerate}
    }
}
}\end{center}
}

In principle, we can extend the above framework, combining spectral and $\ell_1$ approaches, for learning other models on $h$. For instance, when the third order moments of $h$ are sufficient statistics (e.g. when $h$ is a graphical model with treewidth two), it suffices to learn the third order moments of $h$, i.e. $\Ebb[h \otimes h \otimes h]$, where $\otimes $ denotes the outer product of vectors. This can be accomplished as follows:     first employ $\ell_1$ based approach to learn the topic-word matrix $A$, then consider the third order observed moments tensor $T:=\Ebb[x_1 \otimes x_2 \otimes x_3]$. We have that
\[T(A^{\dagger}, A^{\dagger}, A^{\dagger}) =\Ebb[h \otimes h \otimes h],\] where $T(A^{\dagger}, A^{\dagger}, A^{\dagger})$ denotes the multi-linear map of $T$ under $A^\dagger$. 
For details on multi-linear transformation of tensors, see~\cite{AGHKT}.

%

\subsection{Learning using second-order moments}

In Theorem~\ref{thm:eff-depth}, we prove identifiability and learning of hidden DAGs from   second and third order observed moments.  A natural question is what can be done if only the second order moment is provided. The following remark states that if an oracle gives a topological ordering of the DAG structure then the model can be learned only through the second order moment and there is no need to the third order moment.

\begin{remark}\label{rem:eff-second}
A topological ordering of a DAG is a labeling of the nodes such that, for every directed edge $(j,i)$, we have $j < i$. It is a well known result in graph theory that a directed graph is a DAG if and only if it admits a topological ordering. Now, consider a DAG model with a full column rank coefficient matrix $A$ between the observed and hidden nodes. Further, suppose that an oracle provides us with a topological ordering of the induced DAG on the hidden nodes, \emph{i.e.}, for any labeling of the hidden nodes the oracle returns a permutation of the labels which is faithful to a topological ordering of the DAG. Then, the DAG model (matrices $A$ and $\Lambda$) are identifiable from only the second order moment $\Pairs$.
\end{remark}

Remark~\ref{rem:eff-second} is proved in Appendix~\ref{app:eff-second}.

%% file: single-view.tex
\section{Extension to general  linear  (single view) models}\label{sec:singleview}

We have so far described a framework for identifiability and learning of topic models under expansion conditions. In fact, the developed framework holds for any  linear {\em multi-view} model. Recall that if $x_1, x_2, \ldots$ are the words in the document, and $h$ is the topic mixture variable, we have linearity $ \Ebb[x|h ]= A h,$ and multiple (exchangeable and non-degenerate) views corresponding to different words in the document. In particular, the cross-moments between two different words $x_1$ and $x_2$, given $h$, is $\Ebb[x_1 x_2^\top|h ] = A h h^\top A^\top.$

We now extend the results to a general framework where, unlike topic models, only a single observed view is available, and further assumptions are needed to learn in this setting.

Consider an observed random vector $x \in \R^n$ and a hidden random vector $h \in \R^k$. Let $\G= (\vobs \cup \vhid,\edges)$ denote the bipartite graph  with observed nodes $\vobs =
\{x(1),\dotsc,x(n)\}$ and hidden nodes $\vhid = \{h(1),\dotsc,h(k)\}$.
Let $\veps(i)$ be the noise variable associated with $x(i)$, for $i=1,\dotsc,n$
and denote the variance of $\veps(i)$ by $\sigma_{\veps(i)}^2>0$.
Throughout  we use the notation $h := (h(1),\dotsc,h(k))$, $x := (x(1),\dotsc,x(n))$ and $\veps := (\veps(1), \dotsc, \veps(n))$.
The noise terms $\veps$ are assumed to be pairwise uncorrelated.
The class of models considered are specified by the following assumptions.

\begin{condition}[Linear model] \label{cond:linear}
The observed and hidden variables obey the model\footnote{Without loss of generality, assume that $x(i)$,
$\veps(i)$, $h(j)$ are all zero mean.}
\begin{align}
x(i) &= \sum_{j \in \PA_i} a_{ij} h(j) + \veps(i), \quad \text{for } i \in [n]\,, \label{eqn:obs_eqn}
\end{align}
where $\{\veps(i)\}_{i\in[n]}$ are pairwise uncorrelated and are independent from $\{h(j)\}_{j\in[k]}$.
Furthermore, the  matrix $A:=[a_{i,j}]\in \R^{n \times k}$ has full column rank and the hidden variables are linearly independent, \emph{i.e.}, with probability one, if $\sum_{i\in [k]} \alpha_i h(i) = 0$, then $\alpha_i = 0$, for all $i \in [k]$.
\end{condition}

Notice that the structure of $\G$ is defined by the non-zero coefficients in Eq.~\eqref{eqn:obs_eqn}. Therefore, there is no edge among the observed nodes. We define $A \in \R^{n \times k}$ by letting the $(i,j)$ entry be $a_{ij}$ if $ j \in \PA_i$ and zero otherwise.
We refer to matrix $A$ as the \emph{coefficient matrix}.

The above setting is prevalent in a number of applications such as the blind deconvolution of sound and images~\cite{levin2009understanding}. The independent component analysis (ICA) is a special case of the above setting, where the sources $h_i$ are assumed to be independent. In contrast, in our setting, we allow for arbitrary distribution on $h$, and assume expansion (and rank) conditions on the coefficient matrix $A$.

Recall that in case of the topic models, $A$ corresponds to the topic-word matrix. Moreover, in the topic model setting, no assumption is made on the noise variables $\veps$, since the presence of cross-moments (between different words) enables us to remove the dependence on $\veps$. However, in the single view case the second order observed moment $\Sigma : = \E[xx^\T]$
is given by
\begin{eqnarray*}
\Sigma = A \E[hh^\T] A^\T + \E[\veps \veps^\T]\,.
\end{eqnarray*}

\aj{\begin{remark}
The linear relationships described above can be thought of as linear structural equation models (SEM). In general, an SEM is defined by a collection of equations
\begin{eqnarray}\label{eqn:SEM}
z_i = f_i(z_{\PA_i}, \veps_i),
\end{eqnarray}
with $z_i$ be the variables associated to the nodes. Recently, there has been some progress on the identifiability problem of SEMs in the fully observed setting~\cite{Shimizu-06,Hoyer-09, Peters-11,Peters-12}. This paper can be viewed as a contribution to the problem of identifiability and learning SEMs with latent variables.
\end{remark}
}

We now discuss a rank condition on the coefficient matrix $A$, which allows us to remove the noise term $\E[\veps \veps^\T]$
from the second order moment $\Sigma$.

\begin{condition}[Rank condition] \label{cond:rank}
There exists a fixed partition $\Part$ of $[n]$ such that $|\Part| = 3$, and $A_I$ has full column rank for all $I \in \Part$.
\end{condition}

Since $\rank(A_I) = k$, for $I \in \Part$, we have as a consequence $n \ge |\Part|\, k = 3k$. Therefore, it essentially states that the number of hidden nodes should be at most one third of the observed ones. In most applications, we are looking for \emph{a few} number of hidden effects that can represent the statistical dependence relationships among the observed nodes.
Thus the rank condition is reasonable in these cases.

%% file: matrix-decomposition.tex
\subsection{Matrix decomposition method for denoising}

We now show that under the rank assumption  in Condition~\ref{cond:rank},   we can extract the noise terms $\veps$ from the observed  moments through a matrix decomposition method.

 Find a partition $\Part$ of $[n]$, such that $|\Part| = 3$,
and $\rank(\Sigma_{I,J}) = k$ for all distinct $I,J \in \Part$. (Note that
$\rank(\Sigma_{I,J}) = \rank(A_I \E[hh^\T] A_J^\T)$ and by rank condition,
there exists such a partition $\Part$). We now show that the matrix decomposition procedure $\Subroutine(\Sigma, \Part)$ returns $A \E[hh^\T] A^\T$ and the diagonal matrix $\E[\veps \veps^\T]$.
\begin{lemma}\label{lem:block-diagonal}
Let $C = AB^\T + D$, with $A, B \in \R^{n \times k}$ and $D \in \R^{n\times
n}$ a diagonal matrix. Suppose that for a fixed partition $\Part$ of $[n]$,
with $|\Part| = 3$, all the submatrices $A_I$ and $B_I$ have full column
rank $k$, for all $I \in \Part$. Then, $\Subroutine(C)$ returns $AB^\T$ and $D$.
\end{lemma}
The proof of Lemma~\ref{lem:block-diagonal} is deferred to Appendix~\ref{app:block-diagonal}.

\begin{algorithm}
\caption*{$\Subroutine$: Decomposition of a matrix into its low-rank and
diagonal parts.}
\begin{algorithmic}[1]

\REQUIRE Matrix $C = AB^\T + D$, with $A, B \in \R^{n\times k}$, $D \in
\R^{n \times n}$ diagonal, and partition $\Part$ of $[n]$.

\ENSURE Diagonal part $D$ and low-rank part $L = AB^\T$.

\FOR{each $I \in \Part$}

\STATE Choose distinct $J,K \in \Part \backslash \{I\}$.

\STATE Let $U_I\in \R^{|I| \times k}$ be the matrix of left singular
vectors of $C_{I,J}$.

\STATE Let $V_J \in \R^{|J| \times k}$ be the matrix of right singular
vectors of $C_{I,J}$.

\STATE Let $U_K\in \R^{|K| \times k}$ be the matrix of left singular
vectors of $C_{K,J}$.

\STATE Set $A_IB_I^\T = C_{I,J} V_J (U_K^\T C_{K,J} V_J)^{-1} U_K^\T
C_{K,I}$.

\STATE Set $D_{I,I} = C_{I,I} - A_{I}B^\T_{I}$.

\ENDFOR

\RETURN $D$ and $L = C-D$.
\end{algorithmic}
\end{algorithm}

\subsubsection{Remark on finding the partition $\Part$}\label{subsec:partition}
The rank condition for matrix $A$ in Condition~\ref{cond:rank} ensures the existence of a partition $\Part$ of $[n]$, such that, $|\Part|=3$ and $A_I \in \R^{n \times k}$ has full column rank for all $I \in \Part$. However, we are not provided with such a partition. 
We now show that under an \emph{incoherence} assumption about $A$, a random partitioning of its rows into three groups has the desired property, with fixed positive probability.
\begin{definition}
Let $A = USV^\T$ be a thin singular value decomposition of $A$, where $U\in \R^{n\times k}$has orthonormal columns, $S = \diag(\sigma_1(A), \dotsc, \sigma_k(A))$, and $V \in \R^{k \times k}$ is orthogonal. Define the \emph{incoherence number} of $A$ as:
\begin{eqnarray}
c_A : = \max_{j \in [n]} \bigg\{\frac{n}{k} \|U^\T e_j\|_2^2\bigg\}.
\end{eqnarray}
\end{definition}

\begin{lemma} \label{lemma:partitioning}
Fix $\ell \in [n]$, and consider $\ell$ random submatrices $A_1, A_2, \dotsc, A_\ell$ of $A$
obtained by the following process: for each row of $A$, independently
choose one of the $\ell$ submatrices uniformly at random, and put the row
in that submatrix. Fix $\delta \in (0,1)$. Then,
\begin{eqnarray}
\P\Big\{\sigma_k(A_v) \geq \sigma_k(A) / (2\sqrt{\ell}), \forall v \in [\ell] \Big\} \geq
1-\delta,
\end{eqnarray}
provided that $c_A\leq \frac{9}{32} \cdot \frac{n}{k\ell \ln \frac{k\ell}{\delta}}$.
\end{lemma}

Lemma~\ref{lemma:partitioning} is proved in Appendix~\ref{app:partitioning}. Using this lemma with $\ell=3$, we obtain the following. For $A\in \R^{n \times k}$ with full column rank and a random partitioning $\Part$ of its rows into three groups, all the submatrices $A_I$, $I \in \Part$ are full rank with probability at least $1-\delta$, provided that
\begin{eqnarray}
c_A \le \frac{3}{32} \cdot \frac{n}{k\ln \frac{3k}{\delta}}.
\end{eqnarray}

Thus, we have a procedure for denoising (i.e. recovering the noise terms $\veps$) through random partitioning and matrix decomposition under appropriate rank condition. The coefficient matrix $A$ can now be extracted from the denoised moments through the procedures listed in the previous sections, under expansion condition~\ref{cond:expansion} and generic parameters condition~\ref{cond:parameter} for the coefficient matrix $A$.  

%% file: hierarchical.tex
\subsection{Application: learning hierarchical models}\label{subsec:hier}

In the previous section, we developed a general framework for learning linear models with hidden variables. 

We now apply the above results for learning hierarchical models, which consist of many layers of hidden variables. We first formally define hierarchical linear models.
\begin{definition}
A hierarchical linear model is a model with the following graph structure. The nodes of the graph can be partitioned into levels $L_1,\dotsc, L_m$ such that there is no edge between the nodes within one level and all the edges are between nodes in adjacent levels, $(L_i, L_{i+1})$ for $i \in [m-1]$. Furthermore,  the edges are directed from $L_i$ to $L_{i+1}$. The nodes in level $L_m$ correspond to the observed nodes and other levels contain the hidden nodes.
\end{definition}

The next theorem concerns identifiability of linear hierarchical models. More
specifically, consider a hierarchical model and let $\G_i$ be the
induced graph with nodes $L_i \cup L_{i+1}$ and suppose that the induced
model between levels $L_i$ and $L_{i+1}$ satisfies the model conditions
described in Section~\ref{subsec:model} with coefficient matrix $A_i$, for
$i\in [m-1]$: $A_i$ has the rank condition (Condition~\ref{cond:rank}) and
parameter genericity property (Condition~\ref{cond:parameter}), and
(bipartite) graph $\G_i$ has the expansion property
(Condition~\ref{cond:expansion}).

\begin{theorem}\label{thm:multi}
Consider a hierarchical  model with levels $L_1, \dotsc, L_m$ and suppose that the induced model between levels $L_i$ and $L_{i+1}$ satisfies the model conditions described in Section~\ref{subsec:model} with coefficient matrix $A_i$, for $i \in [m-1]$. Then all columns of $A_i$ are identifiable for $i\in [m-1]$ from the second order observed moment, i.e., $\Sigma = \E[x x^\T]$. Therefore, the entire model is identifiable up to permuting the nodes within each level.
\end{theorem}

Theorem~\ref{thm:multi} is proved in Section~\ref{subsec:multi-proof}.

\begin{remark}
By the definition of a hierarchical  model, the hidden nodes in level $L_1$ are independent. Now consider the case that the nodes in $L_1$ have arbitrary dependence relationships. By using the same argument as in the proof of Theorem~\ref{thm:multi}, we can still learn all the coefficient matrices $A_i$ and the second order moment of the variables in layer $L_1$.
\end{remark} 

%% file: simulation-dag.tex
\section{Numerical experiments}\label{sec:simulation}
\aj{Should explicitly mention that simulations are for single view ... I will do that once you are done with single view section.}

In the previous sections, we proposed algorithms for learning topic models (multi-view), and general linear single view models. 
Our algorithms rely on low order (second and third order) moments of the observed variables. In presenting the results and the proofs we
assumed that exact observed moments are available to emphasize the validity of the method. In general, these moments should be
estimated from sampled data. This brings up the question of \emph{sample complexity}, namely given a model $\G$, how many samples
are required to estimate the model parameters with precision $\delta$. We expect graceful sample complexity for the proposed algorithms as the low order moments can be reliably estimated from data. In this section, we consider two concrete examples of the \emph{single view} linear models, and validate the performance of the proposed algorithms under finite number of samples.

The first example is a hierarchical model where we require the coefficient matrices between adjacent layers
to be full rank. The second example is an illustration of a model in which the relations among the hidden nodes are described by a (general) DAG, and we require the
coefficient matrix to be full rank.
\smallskip\\

\noindent{\bf Example 1.} We validate our method on the following configuration.
\begin{itemize}
\item{\emph{Graph structure:}}
 We consider a hierarchical model with three levels, $L_1$, $L_2$ and $L_3$. Levels $L_1$ and $L_2$ contain the hidden nodes with $n_1 = |L_1| = 5$, $n_2 = |L_2| = 30$ and level $L_3$ contains the observed nodes with $n_3 = |L_3| = 180$. Coefficient matrices $A_1\in \R^{n_2\times n_1}$ and $A_2\in \R^{n_3\times n_2}$, respectively representing the linear relationships among the levels $L_1$, $L_2$ and the levels $L_2$, $L_3$, are constructed according to a Bernoulli-Gaussian model.  More specifically, $A_1 = B \odot G$, where $B \in \R^{n_2\times n_1}$ is an i.i.d. Bernoulli$(p)$ matrix, and $G\in \R^{n_2\times n_1}$ has i.i.d. standard normal entries. Further, $\odot$ indicates the entrywise product. In our experiment, we choose $p = 0.3$ to make the model satisfy the expansion property. Also recall that Theorem~\ref{thm:alg1} assumes  a positive gap $\gamma_i$ between the maximum and the second maximum absolute values in the $i^{th}$ row, for $i\in [n]$. For the sake of simplicity, we consider the same gap $\gamma$ for all the rows. More specifically, in each row of $A_1$ we change the entry with the maximum absolute value to ensure gap $\gamma$ while keeping the sign of this entry unchanged. As we will see, $\gamma$ has an important effect on sample complexity of the algorithm. A very small $\gamma$ leads to a poor sample complexity and increasing $\gamma$ improves the sample complexity of the algorithm. Similar model is used to generate $A_2$.

\item{\emph{Noise variables:}}
For each noise variable, its variance is selected uniformly at random from the interval $[0.5,1]$ and its distribution is chosen from a family of four distributions including $(a)$ exponential; $(b)$ poisson; $(c)$ chi-squared; $(d)$ gaussian. More specifically, for given variance $\sigma^2$,
it is distributed as either $\expo(\sigma^{-1})$, $\Poiss(\sigma^2)$,  $(\sigma/\sqrt{2}) \chi_1$, or $\normal(0,\sigma^2)$ equally likely,
where $\chi_1$ denotes the chi-squared distribution with one degree of freedom.
\end{itemize}

In experiments we employed  a slight variant of $\AlgI$ to make it more robust to finite sample errors.
This is essentially the same variant of {\sf ER-SpUD} used in~\cite{Spielman-12} (see {\sf ER-SpUD (proj)}). For self-containedness, we
present its details in Appendix~\ref{app:AlgIproj}.


\aj{Should be modified based on the single view Section! I will do it ...}

Following Lemma~\ref{lemma:partitioning}, we find partition $\Part$ (the input of $\Subroutine$) 
by randomly partitioning the rows of the corresponding coefficient matrix
into three groups. With exact observed moments, any such partition leads to the decomposition of
the corresponding matrix into its low rank and diagonal parts, with a fixed positive probability. However, with empirical
moments, different partitions lead to different errors in estimating the coefficients. In experiments, we run $\Subroutine$ with $100$ different partitions.
Due to finite sample error, the retuned matrix $D$ for each run is not necessarily a diagonal matrix. We compute the ratio of off-diagonal entries
for each retuned $D$, i.e., $\sum_{i\neq j}|D_{ij}|/ \sum_{i,j}|D_{ij}|$ and choose the partition $\Part$ which leads to the minimum off-diagonal ratio.

We run $\AlgI$ with empirical covariance $\hSigma$ to first learn $A_1$ and then $A_2$ as described in the proof of Theorem~\ref{thm:multi}. More specifically, using $\nsmp$ independent realizations of the observed variables $x^{(1)},\dotsc, x^{(\nsmp)} \in \R^{n_3}$, with $x^{(i)}$ representing the values of the observed nodes in the $i$-th realization, we let
\begin{align*}
m = \frac{1}{\nsmp} \sum_{i=1}^{\nsmp} x^{(i)}\,,\quad \quad
\hSigma = \frac{1}{\nsmp} \sum_{i=1}^{\nsmp}(x^{(i)} - m)(x^{(i)} - m)^\T\,.
\end{align*}
\smallskip

\noindent $\bullet$ \emph{Measure of performance and the results:}
Recall that coefficient matrices can be only specified up to permutation and scaling of their columns. In order to measure the algorithm performance on estimating a coefficient matrix $A \in \R^{n \times m}$, we define the following distance between $A$ and the estimation $\widehat{A}$ returned by the algorithm.
\begin{align*}
\dist(A,\widehat{A}) &= \frac{1}{\|A\|_F^2} \sum_{i=1}^m \min_{j\in[m], \nu}  \|Ae_i - \nu \widehat{A} e_j\|^2\,\\
&= \frac{1}{\|A\|_F^2} \sum_{i=1}^m \min_{j\in[m]}  \|Ae_i - (e_i^\T A^\T \widehat{A}e_j) \widehat{A} e_j\|^2\,.
\end{align*}
Here, the minimization over $j \in [m]$ is to remove the permutation ambiguity and the minimization over $\nu$ is to remove the scaling ambiguity.
Further, since $\AlgI$ returns the matrix in its canonical form, we have $\|\widehat{A}e_j\| =1$ and the optimal $\nu$ is given by $\nu = (e_i^\T A^\T \widehat{A}e_j)$.

The support of coefficient matrix $A$ corresponds to the edges in the corresponding graph and is of particular interest.
We define precision and recall in characterizing the support of $A$ as follows:
\begin{align*}
\text{precision}(A,\widehat{A}) = \frac{|\{(i,j): A_{ij} \neq 0, \widehat{A}_{ij} \neq 0\}|}{|\{(i,j): \widehat{A}_{ij} \neq 0\}|}\,,
\quad \,\,\,
\text{recall }(A,\widehat{A}) = \frac{|\{(i,j): A_{ij} \neq 0, \widehat{A}_{ij} \neq 0\}|}{|\{(i,j): {A}_{ij} \neq 0\}|}\,.
\end{align*}
In words, \emph{precision} is the fraction of retrieved edges that are truly an edge and \emph{recall} is the fraction of true edges that are retrieved.

We summarize the results in Table~\ref{tbl:multi-simulation} for different values
of $\gamma$ and $\nsmp$.
\begin{table}[h!]
\centering
\subtable[$\gamma = 0.3$]{
\small
\centering
\begin{tabular}{|c|ccccc|}
\hline
$\nsmp$ & {\bf 25,000} & {\bf 50,000} & {\bf 100,000} & {\bf 200,000} & {\bf 400,000}
\\
\cline{1-6}
\multicolumn{1}{|c|}{$\dist(A_1,\widehat{A}_1)$} & 0.9283 &  0.8029& 0.7656 &  0.6939 & 0.4813
\\
\cline{1-6}
\multicolumn{1}{|c|}{precision$(A_1,\widehat{A}_1)$} & 0.3120 &  0.3228 & 0.3231 & 0.3325 &  0.3333
\\
\cline{1-6}
\multicolumn{1}{|c|}{recall$(A_1,\widehat{A}_1)$} & 0.8478 & 0.8913 & 0.9130 &0.9130 &  0.9130
\\
\cline{1-6}
\multicolumn{1}{|c|}{$\dist(A_2,\widehat{A}_2)$} & 0.2674&  0.1516 & 0.1466 & 0.1299 & 0.0943
\\
\cline{1-6}
\multicolumn{1}{|c|}{precision$(A_2,\widehat{A}_2)$} &0.3355 &  0.3402&0.3497  &  0.3530 &  0.3566
\\
\cline{1-6}
\multicolumn{1}{|c|}{recall$(A_2,\widehat{A}_2)$} &0.9389 &  0.9518 &  0.9526 &  0.9599 &0.9697
\\
\cline{1-6}
\end{tabular}
\label{tbl:multi-simulation1}
}
\subtable[$\gamma = 0.5$]{
\small
\centering
\begin{tabular}{|c|ccccc|}
\hline
$\nsmp$ & {\bf 25,000} & {\bf 50,000} & {\bf 100,000} & {\bf 200,000} & {\bf 400,000}
\\
\cline{1-6}
\multicolumn{1}{|c|}{$\dist(A_1,\widehat{A}_1)$} & 0.5942 &  0.4016& 0.3205 &  0.1187 & 0.0661
\\
\cline{1-6}
\multicolumn{1}{|c|}{precision$(A_1,\widehat{A}_1)$} &0.3462 & 0.3462 & 0.3538 &  0.3615& 0.3769
\\
\cline{1-6}
\multicolumn{1}{|c|}{recall$(A_1,\widehat{A}_1)$} & 0.8824 & 0.8824 & 0.9020 & 0.9216 & 0.9608
\\
\cline{1-6}
\multicolumn{1}{|c|}{$\dist(A_2,\widehat{A}_2)$} &0.0731 &  0.0338 & 0.0157 & 0.0084 & 0.0048
\\
\cline{1-6}
\multicolumn{1}{|c|}{precision$(A_2,\widehat{A}_2)$} & 0.3437&  0.3497&  0.3552&0.3558  &  0.3581
\\
\cline{1-6}
\multicolumn{1}{|c|}{recall$(A_2,\widehat{A}_2)$} & 0.9477 &  0.9641 &  0.9793& 0.9811 & 0.9872
\\
\cline{1-6}
\end{tabular}
\label{tbl:multi-simulation2}
}
\caption{Example 1. Hierarchical (single-view) model with level sizes $n_1 = 5, n_2 = 30, n_3 = 180$, and $1694$ number of edges.}\label{tbl:multi-simulation}.
\end{table}
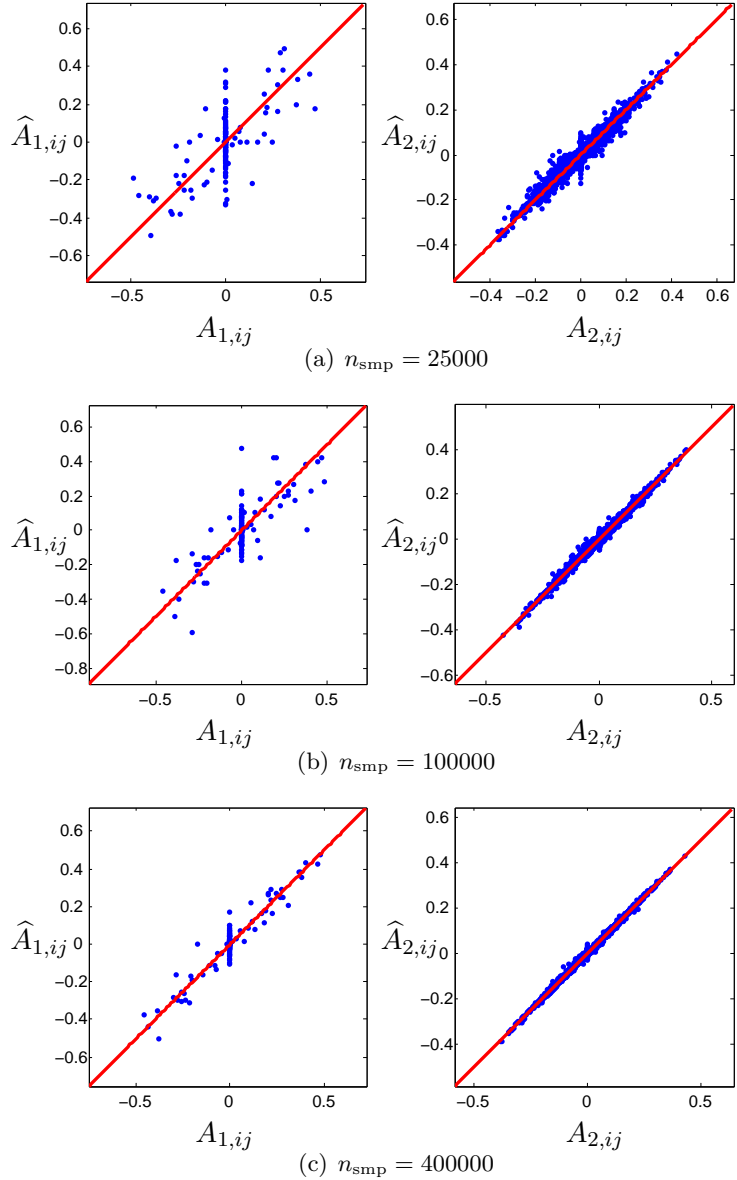
\begin{figure}[!t]
\centering
\subfigure[$\nsmp = 25000$]{
\includegraphics*[viewport= 90 290 525 485, width=3.8in]{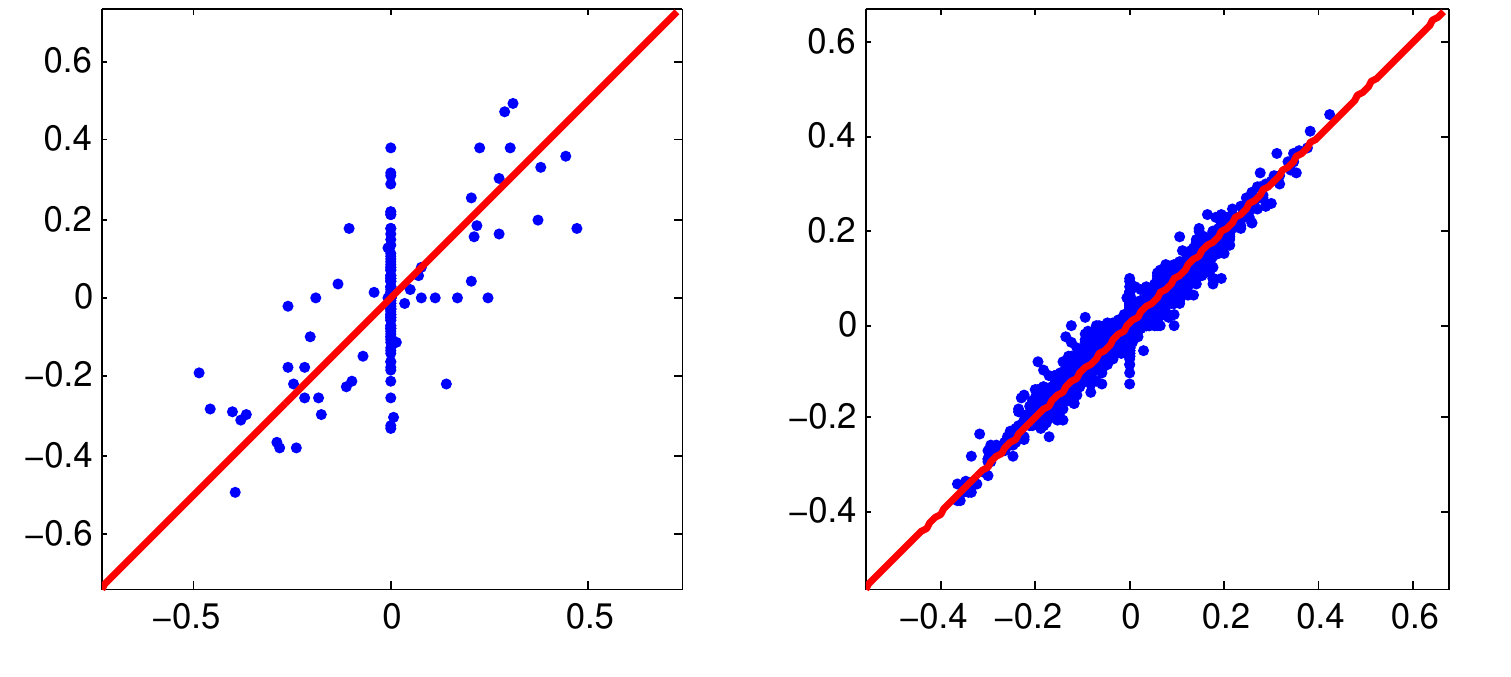}
\put(-215,-5){$A_{1,ij}$}
\put(-285,68){$\widehat{A}_{1,ij}$}
\put(-75,-5){$A_{2,ij}$}
\put(-145,68){$\widehat{A}_{2,ij}$}
}
\vspace{0cm}
\subfigure[$\nsmp = 100000$]{
\includegraphics*[viewport= 90 290 525 485, width=3.8in]{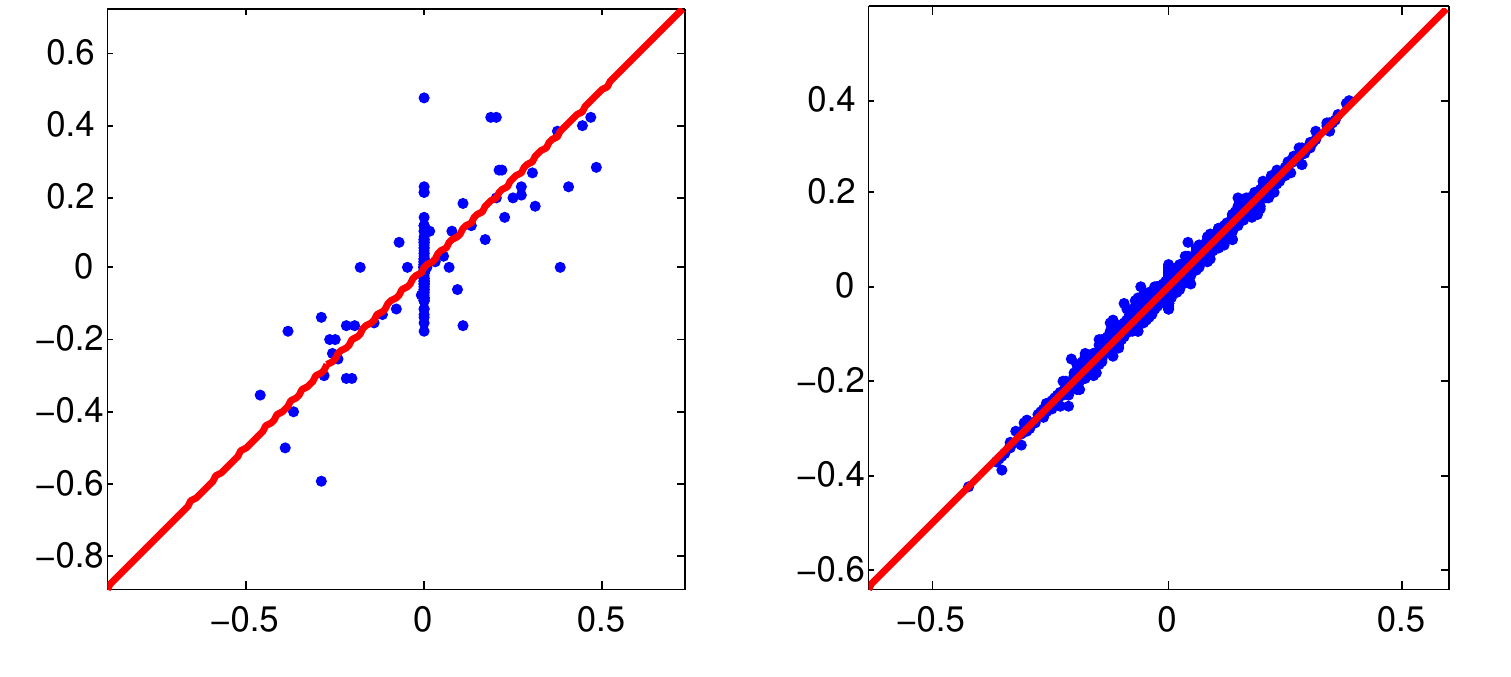}
\put(-215,-5){$A_{1,ij}$}
\put(-285,68){$\widehat{A}_{1,ij}$}
\put(-75,-5){$A_{2,ij}$}
\put(-145,68){$\widehat{A}_{2,ij}$}
}
\vspace{0cm}
\subfigure[$\nsmp = 400000$]{
\includegraphics*[viewport= 90 290 525 485, width=3.8in]{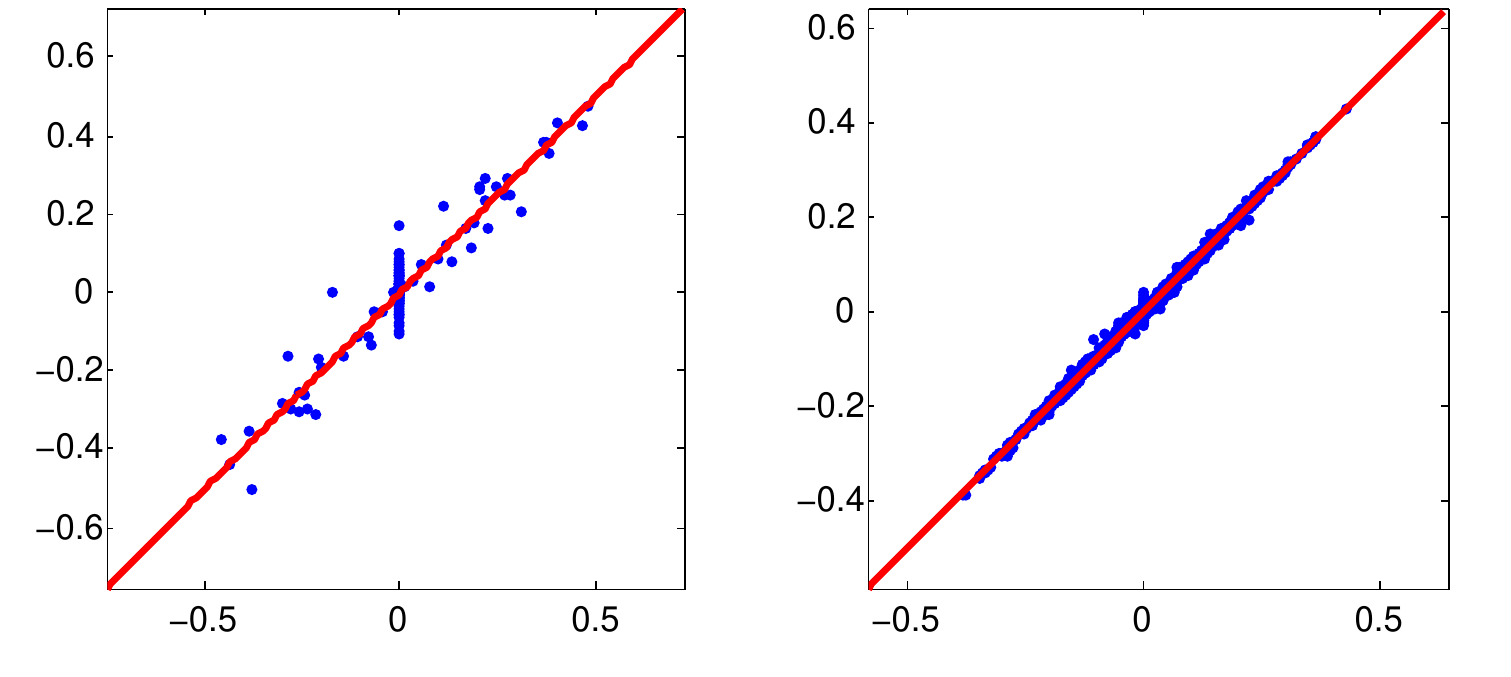}
\put(-215,-5){$A_{1,ij}$}
\put(-285,68){$\widehat{A}_{1,ij}$}
\put(-75,-5){$A_{2,ij}$}
\put(-145,68){$\widehat{A}_{2,ij}$}
}
\caption{Scatterplots for Learning the hierarchical model in Example 1, using different values of $\nsmp$ and $\gamma = 0.5$.}\label{fig:multi-simulation}
\vspace{-0.5cm}
\end{figure}
The scatterplots in Fig.~\ref{fig:multi-simulation} depict the points $(\widehat{A}_{1,ij}, A_{1,ij})$ and $(\widehat{A}_{2,ij}, A_{2,ij})$ for different values of $\nsmp$
and $\gamma = 0.5$. As the number of samples increases, the observed moments are
estimated more accurately and the scatter points concentrate around the line with slope one.
Further, for each value of $\nsmp$, the error in estimating $A_1$ is larger than the error in estimating $A_2$. The reason is that we first apply $\AlgIproj$ to estimate the coefficient matrix $A_2$,
and then use this estimation to learn the coefficient matrix $A_1$. In other words the induced model between the observed nodes (level $L_3$) and the hidden nodes (level $L_2$) is estimated more accurately than the induced model among the hidden nodes (levels $L_1, L_2$).
\smallskip\\

\noindent{\bf Example 2.}  Our next example is a model in which the relationships among the hidden nodes
 are represented by a DAG model. The model contains $k = 25$ hidden nodes
and $n = 150$ observed nodes. The linear relationships among the hidden nodes are described by
a lower triangular coefficient matrix $\Lambda \in \R^{k\times k}$, which is chosen according to a Bernoulli-Gaussian model: The entries in the lower triangular part are non-zero with probability $p = 0.3$ and the values of the non-zero entries are chosen independently
from standard normal distribution. The coefficient matrix $A\in \R^{n\times k}$, describing the relationships between the hidden
nodes and the observed nodes, is constructed as per Bernoulli-Gaussian model in the previous experiment with $p= 0.3,$ and then ensured to have gap $\gamma$ between the maximum and the second maximum absolute values in each row.

Similar to the previous experiment, the noise variables have variances chosen uniformly at random from $[0.5,1]$. Their distributions are chosen uniformly at random from a family of three distributions with non-zero skewness, namely $(a)$ exponential; $(b)$ poisson; $(c)$ chi-squared.

In simulations, we used the power iteration to implement the ECA part as described in Section~\ref{subsec:BN}. 
%

The results are summarized in Table~\ref{tbl:depthone-simulation}.
The scatterplots in Fig.~\ref{fig:depthone-simulation} contains the points $(\widehat{\Lambda}_{ij}, \Lambda_{ij})$ and $(\widehat{A}_{ij}, A_{ij})$ for different values of $\nsmp$
and $\gamma = 0.5$.
\begin{table}[]
\centering
\subtable[$\gamma = 0.3$]{
\small
\centering
\begin{tabular}{|c|cccc|}
\hline
$\nsmp$ & {\bf 200,000} & {\bf 300,000} & {\bf 400,000} & {\bf 500,000}
\\
\cline{1-5}
\multicolumn{1}{|c|}{$\dist(\Lambda,\widehat{\Lambda})$} &0.7933  &  0.4627 &0.3894  &  0.1778
\\
\cline{1-5}
\multicolumn{1}{|c|}{precision$(\Lambda,\widehat{\Lambda})$} &0.1168  &0.1168  & 0.1168 & 0.1168
\\
\cline{1-5}
\multicolumn{1}{|c|}{recall$(\Lambda,\widehat{\Lambda})$} & 1 & 1& 1 & 1
\\
\cline{1-5}
\multicolumn{1}{|c|}{$\dist(A,\widehat{A})$} & 0.2818 & 0.2584 &0.1894  &   0.0809
\\
\cline{1-5}
\multicolumn{1}{|c|}{precision$(A,\widehat{A})$} & 0.2979&  0.3248 & 0.3263&  0.3337
\\
\cline{1-5}
\multicolumn{1}{|c|}{recall$(A,\widehat{A})$} & 0.9391&  0.9446 &  0.9492 &  0.9705
\\
\cline{1-5}
\end{tabular}
\label{tbl:depthone-simulation1}
}
\subtable[$\gamma = 0.5$]{
\small
\centering
\begin{tabular}{|c|cccc|}
\hline
$\nsmp$ & {\bf 200,000} & {\bf 300,000} & {\bf 400,000} & {\bf 500,000}
\\
\cline{1-5}
\multicolumn{1}{|c|}{$\dist(\Lambda,\widehat{\Lambda})$} & 0.4597 &0.1820  & 0.0832 &  0.0492
\\
\cline{1-5}
\multicolumn{1}{|c|}{precision$(\Lambda,\widehat{\Lambda})$} &0.1168 &  0.1168&  0.1168&  0.1168
\\
\cline{1-5}
\multicolumn{1}{|c|}{recall$(\Lambda,\widehat{\Lambda})$} & 1 &  1& 1 & 1
\\
\cline{1-5}
\multicolumn{1}{|c|}{$\dist(A,\widehat{A})$} & 0.1777 &  0.0757&  0.0478 & 0.0330
\\
\cline{1-5}
\multicolumn{1}{|c|}{precision$(A,\widehat{A})$} &0.3283 &  0.3302& 0.3333 &  0.3352
\\
\cline{1-5}
\multicolumn{1}{|c|}{recall$(A,\widehat{A})$} & 0.9548 & 0.9603  & 0.9695 & 0.9751
\\
\cline{1-5}
\end{tabular}
\label{tbl:depthone-simulation2}
}
\caption{Example 2. Bayesian network (single-view) model with $k = 25$ hidden nodes, $n = 150$ observed nodes, and $1177$ number of edges.}\label{tbl:depthone-simulation}.
\end{table}
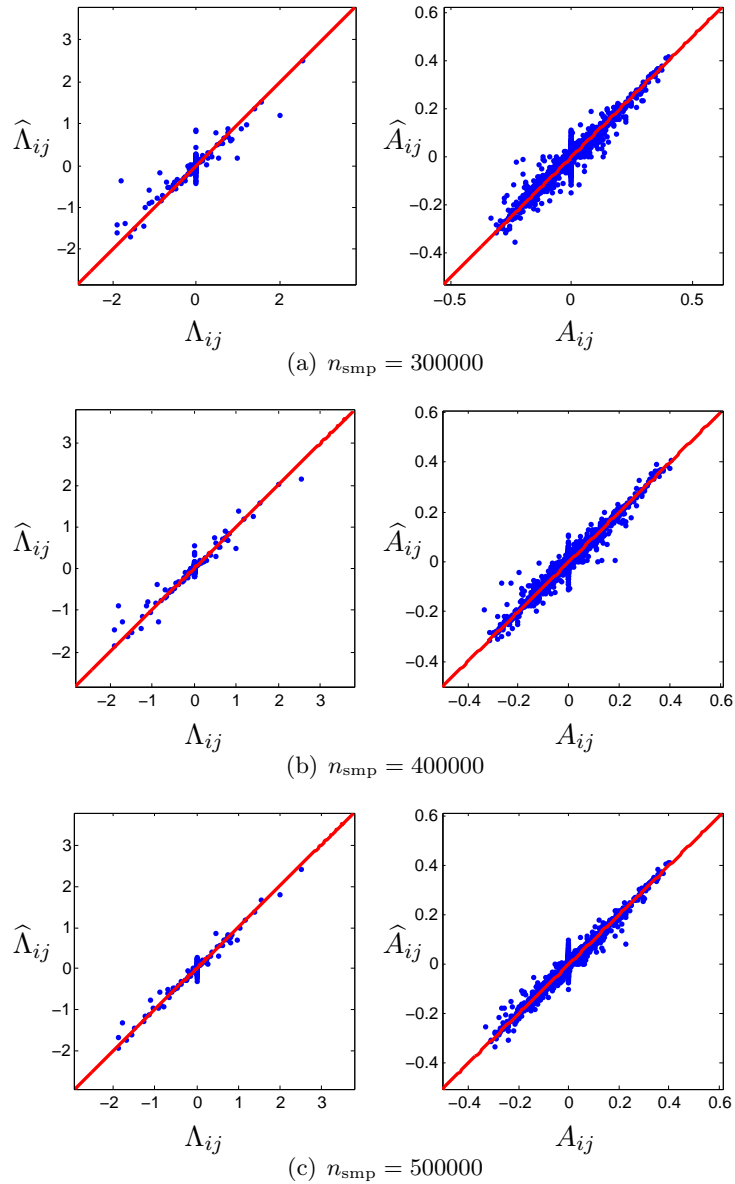
\begin{figure}[!t]
\centering
\subfigure[$\nsmp = 300000$]{
\includegraphics*[viewport= 90 290 525 485, width=3.8in]{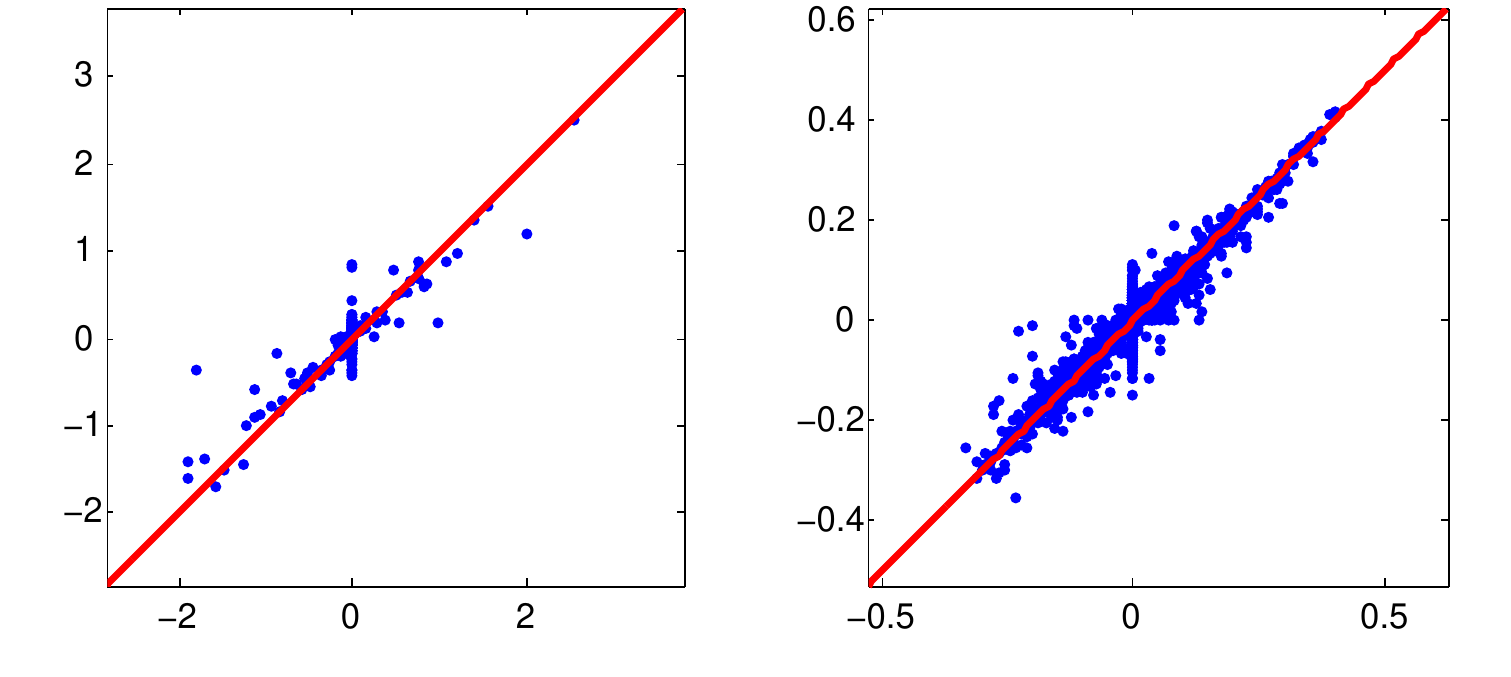}
\put(-215,-5){$\Lambda_{ij}$}
\put(-280,68){$\widehat{\Lambda}_{ij}$}
\put(-75,-5){$A_{ij}$}
\put(-140,68){$\widehat{A}_{ij}$}
}
\vspace{0cm}
\subfigure[$\nsmp = 400000$]{
\includegraphics*[viewport= 90 290 525 485, width=3.8in]{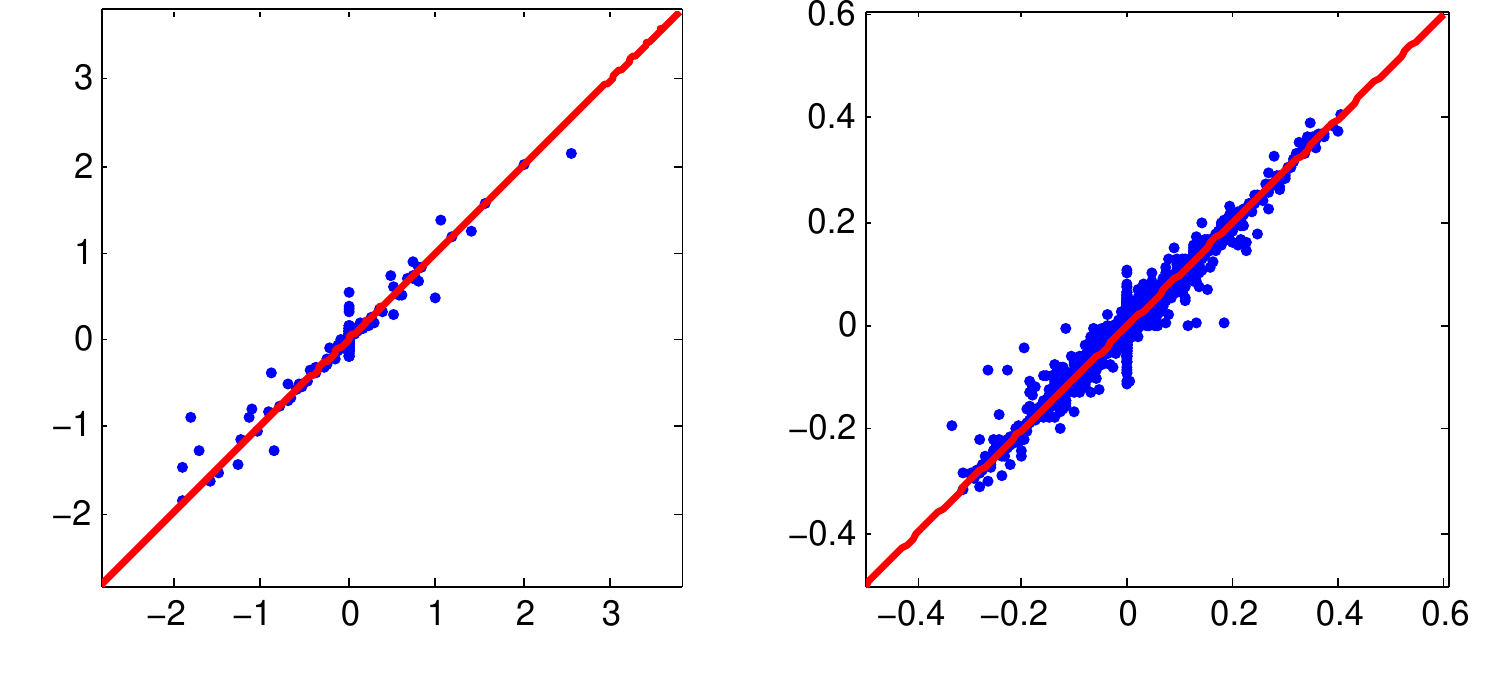}
\put(-215,-5){$\Lambda_{ij}$}
\put(-280,68){$\widehat{\Lambda}_{ij}$}
\put(-75,-5){$A_{ij}$}
\put(-140,68){$\widehat{A}_{ij}$}
}
\vspace{0cm}
\subfigure[$\nsmp = 500000$]{
\includegraphics*[viewport= 90 290 525 485, width=3.8in]{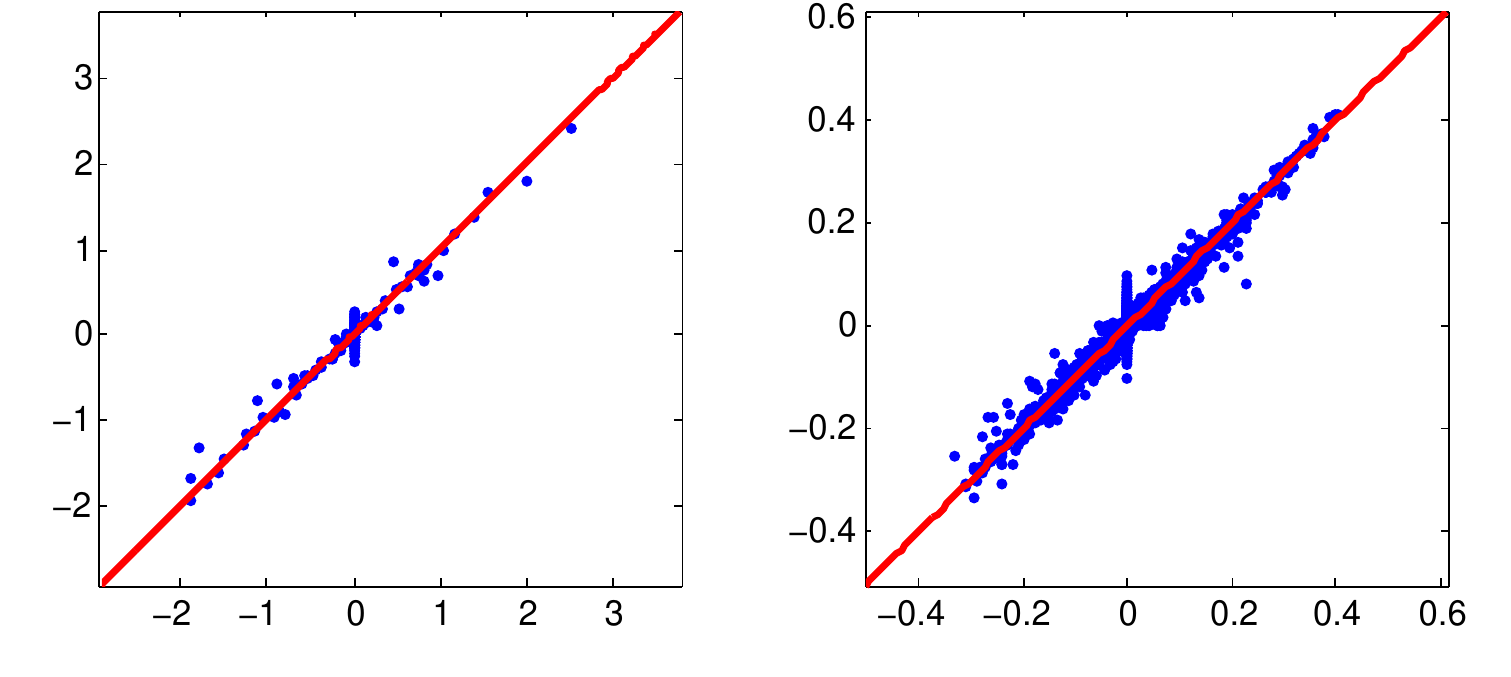}
\put(-215,-5){$\Lambda_{ij}$}
\put(-280,68){$\widehat{\Lambda}_{ij}$}
\put(-75,-5){$A_{ij}$}
\put(-140,68){$\widehat{A}_{ij}$}
}
\caption{Scatterplots for Learning the model in Example 2, using different values of $\nsmp$ and $\gamma = 0.5$.}\label{fig:depthone-simulation}
\end{figure}
%

%% file: proofs.tex
\appendix
\section{Proof of the theorems}
\subsection{Proof of Theorem~\ref{thm:core}}
\label{sec:core-proof}
Observe that
\begin{eqnarray}
\begin{split}
\Pairs &= \E[x_1\otimes x_2] = \E[\E[x_1\otimes x_2|h]] = A \E[hh^\T]A^\T.
\end{split}
\end{eqnarray}

Since the hidden variables are linearly independent, $\E[hh^\T]$is full
rank. Otherwise, $v^\T \E[hh^\T]v = 0$ for some non-zero vector $v$. This
implies that $\E[\|h^\T v\|^2] = 0$ and so $h^\T v = 0$ which leads to a
contradiction.

Given that $\E[hh^\T]$ and $A$ have full column rank, we have $\Col(A) =
\Col(A \E[hh^\T] A^\T)$. Let $\{u_1, \dotsc, u_k\}$ be any basis of $\Col(A
\E[hh^\T] A^\T)$ containing vectors with $k$ smallest $\ell_0$ norm. Since
all the columns of $A$ have at most $d_{\max}$ non-zero entries, we have
$\max_{i\in [k]} \|u_i\|_0 \le d_{\max}$, by choice of vectors $u_i$. Next
we show that due to the graph expansion property
(Condition~\ref{cond:expansion}) and the parameter genericity property
(Condition~\ref{cond:parameter}), vectors $u_i$ are (scaled) columns of
$A$.
Observe that any vector $u_i$ can be represented by a linear combination of
columns of $A$, say $u_i = Av$. If $\|v\|_0 \ge 2$, then
\begin{eqnarray*}
\|u_i\|_0 = \|Av\|_0 > |\Neigh_A(\supp(v))| - |\supp(v)|\ge d_{\max},
\end{eqnarray*}
where the first inequality follows from parameter genericity property and
the second one follows from the expansion property.
This leads to a contradiction.
Therefore, $\|v\|_0 = 1$, and $u_i$ is scaled multiple of a column of $A$.
Since $\{u_1,\dotsc,u_k\}$ are linearly independent, different $u_i$'s correspond to
different columns of $A$ and therefore columns of $A$, in a canonical form (up to sign),
are given by $\{u_1/\|u_1\|,\dotsc, u_k/\|u_k\|\}$.
%

\subsection{Proof of Theorem~\ref{thm:alg1}}
\label{subsec:alg1-proof}

Recall that $\Pairs  = A \E[hh^\T] A^\T$.
Using following lemma (with $L = \Pairs^{1/2}$) shows that vectors $s_i$, returned by the first loop (steps $(1)-(3)$), are scaled multiples of the columns of $A$.

\begin{lemma}\label{lem:alg1}
Let $A \in \R^{n\times k}$ be a given matrix with rank $k$, and let $L \in \R^{n\times k}$ be such that $L = A M$, for an invertible $M \in \R^{k \times k}$. (Equivalently $\Col(A) = \Col(L)$). Fix $i \in [n]$ and consider the following optimization problem:
\begin{eqnarray}\label{eqn:opt1-alg1}
\min_w \quad \|Lw\|_1 \quad \quad \text{ subject to }(e_i^\T L)w = 1.
\end{eqnarray}
Under the following conditions, $s_i = L w$ is a scaling of the $\pi_i(1)$-th column of $A$. (Recall that $\pi_i(1)$ is the index of the entry with maximum absolute value in the $i$-th row of $A$).

\begin{itemize}
\item [(i)]
$\|A_{(\Neigh^2_{i})^c, (\Neigh_i)^c}\, v\|_1 >
\|A_{\Neigh^2_{i},(\Neigh_i)^c}\, v\|_1$
for all non-zero vectors $v \in \R^{|(\Neigh_i)^c|}$.

\item[(ii)]
$\|A_{(\Neigh_j)^c, \Neigh_i \backslash j} \, v\|_1 > \|A_{\Neigh_j, \Neigh_i \backslash j}\, v\|_1 + (1-\gamma) \|A_{\Neigh_j,j}\|_1 \|v\|_1$
for all $j \in \Neigh_i$ and
all non-zero vectors $v \in \R^{|\Neigh_i|-1}$.
\end{itemize}
\end{lemma}

\begin{proof}[Proof (Lemma~\ref{lem:alg1})]
Consider the following equivalent formulation of Problem~\eqref{eqn:opt1-alg1} obtained by the change of variables $z = Mw$, $b^\T = (e_i^\T L) M^{-1}$:
\begin{eqnarray}\label{eqn:opt2-alg1}
\min_z \quad \|Az\|_1 \quad \quad \text{ subject to }b^\T z = 1.
\end{eqnarray}
Observe that $b^\T$ is the $i$-th row of $A$. Denote the solution to Problem~\eqref{eqn:opt2-alg1} by $z_*$. We aim to prove that $z_*$ is supported on $\{\pi_i(1)\}$. We prove the desired result in two steps:
\begin{claim}\label{claim:alg1-claim1}
Under Condition $(i)$, we have $\supp(z_*) \subseteq \supp(b)$.
\end{claim}
\begin{claim}\label{claim:alg1-claim2}
Under Condition $(i)-(ii)$, we have $\supp(z_*) = \{\pi_i(1)\}$.
\end{claim}

\begin{proof}[Proof (Claim~\ref{claim:alg1-claim1})]
Notice that $b^\T = e_i^\T A$, and so $\supp(b) = \Neigh_i$.
Define $z_0 \in \R^k$ by $z_0(j) := z_*(j)$ for all $j \in \supp(b)$, and
$z_0(j) := 0$ for all $j \notin \supp(b)$.
Also, let $z_1 := z_* - z_0$.
Therefore, $z_0$ is also a feasible solution to Problem~\eqref{eqn:opt2-alg1},
since $b^\T z_0 = b^\T z_*$.

If $z_1 \neq 0$, then
\begin{align*}
\|Az_*\|_1 &= \|A_{\Neigh_i^2,[k]} \, z_*\|_1 + \|A_{(\Neigh_i^2)^c,[k]}\, z_*\|_1\\
& = \|A_{\Neigh_i^2,[k]} \,(z_0 +z_1)\|_1 + \|A_{(\Neigh_i^2)^c,[k]}\, z_1\|_1\\
&\ge \|A_{\Neigh_i^2,[k]} z_0\|_1 - \|A_{\Neigh_i^2,[k]} z_1\|_1 + \|A_{(\Neigh_i^2)^c,[k]}\, z_1\|_1\\
& = \|Az_0\|_1- \|A_{\Neigh_i^2,[k]} z_1\|_1 + \|A_{(\Neigh_i^2)^c,[k]}\, z_1\|_1\\
& > \|Az_0\|_1,
\end{align*}
where the last inequality follows from Condition (i) and the fact
$\supp(z_1) \subseteq (\Neigh_i)^c$.
Therefore, $z_0$ is a feasible solution with smaller objective value, which
contradicts the optimality of $z_*$.
Therefore we conclude that $z_1 = 0$, and hence $\supp(z_*) \subseteq
\supp(b)$.
\end{proof}

\begin{proof}[Proof (Claim~\ref{claim:alg1-claim2})]
By Claim~\ref{claim:alg1-claim1}, $\supp(z_*) \subseteq \supp(b) = \Neigh_i$.
To lighten the notation, let $j = \pi_i(1)$, and define $z_0 := (e_j^\T z_*) e_j$ and $z_1 := z_*
- z_0$.
Suppose for sake of contradiction that $z_1 \neq 0$.
Since $b^\T z_* = 1$, we have $z_0 = ((1 - b^\T z_1) / b_j)\, e_j$.
Therefore (using the triangle inequality twice),
\begin{align*}
\|Az_*\|_1
& = \|A_{\Neigh_j,[k]} z_*\|_1
+ \|A_{(\Neigh_j)^c,[k]} z_* \|_1
\\
& = \|A_{\Neigh_j,[k]} (z_0 + z_1)\|_1
+ \|A_{(\Neigh_j)^c,[k]} z_1 \|_1
\\
& \geq \|A_{\Neigh_j,[k]} z_0\|_1
- \|A_{\Neigh_j,[k]} z_1\|_1
+ \|A_{(\Neigh_j)^c,[k]} z_1 \|_1
\\
& = \|A_{\Neigh_j,[k]} ((1 - b^\T z_1)/b_j)\, e_j\|_1
- \|A_{\Neigh_j,[k]} z_1\|_1
+ \|A_{(\Neigh_j)^c,[k]} z_1 \|_1
\\
& \geq
(1/|b_j) \|A_{\Neigh_j,[k]} e_j\|_1
- |b^\T z_1/b_j| \|A_{\Neigh_j,[k]} e_j\|_1
- \|A_{\Neigh_j,[k]} z_1\|_1
+ \|A_{(\Neigh_j)^c,[k]} z_1 \|_1.
\end{align*}
Since $z_1(j) = 0$, we have $|b^\T z_1| \leq |b|_{\pi_i(2)} \|z\|_1$ by
H\"older's inequality, and therefore,
\begin{align*}
\frac{|b^\T z_1|}{|b_j|}
& \leq \frac{|b|_{\pi_i(2)}\|z_1\|_1}{|b|_{j}}
\leq (1-\gamma_i) \|z_1\|_1.
\end{align*}
Moreover, by Condition (ii) and the fact $\supp(z_1) \subseteq \Neigh_i
\setminus j$,
\begin{align*}
\|A_{\Neigh_j^c,[k]} z_1\|_1
& > \|A_{\Neigh_j,[k]} z_1\|_1
+ (1-\gamma_i) \|A_{\Neigh_j,j}\|_1 \|z_1\|_1
.
\end{align*}
Putting the last three displayed inequalities together gives
\begin{align*}
\|Az_*\|_1
& > (1/|b_j|) \|A_{\Neigh_j,[k]}e_j\|_1
= \|A(e_j / b_j)\|_1\,.
\end{align*}
Since $e_j / b_j$ is a feasible solution, the above strict inequality
contradicts the optimality of $z_*$.
Therefore we conclude that $z_1 = 0$, and $z_* = z_0 = e_j / b_j$.
\end{proof}

Notice that $s_i = L w = AMw = Az_*$ and since $\supp(z_*) = \{\pi_i(1)\}$, $s_i$ is a scaled multiple of the $\pi_i(1)$-th column of $A$. This completes the proof of Lemma~\ref{lem:alg1}.
\end{proof}

Now, we are ready to prove the theorem.

Given that Conditions $(i)-(ii)$ hold for all $i\in[n]$, using Lemma~\ref{lem:alg1}, the set $\mathcal{S} = \{s_1, \dotsc, s_n\}$ consists of scaled multiples of the columns of $A$. Moreover, since $[k] \subseteq \{\pi_1(1), \dotsc, \pi_n(1)\}$, $\mathcal{S}$ contains at least one scaled multiple of each column of $A$. In the second loop (steps $(4)-(8)$), we choose a linearly independent set $\{v_1, \dotsc, v_k\} \subseteq \mathcal{S}$. These are the (scaled multiples of the) columns of $A$.
Hence, letting $\widehat{A} = [\frac{v_1}{\|v_1\|}| \dotsb|\frac{v_k}{\|v_k\|}]$, there exists a permutation matrix $\Pi$, such that $\widehat{A}\Pi$ gives $A$ in its canonical form (up to sign of each column).

\subsection{Proof of Theorem~\ref{thm:eff-depth}}
\label{subsec:eff-depth-proof}

Let $\eta:= (\eta(1), \dotsc, \eta(k))$ and $\veps := (\veps(1), \dotsc, \veps(n))$. Using the model description, we have
\begin{eqnarray}
\Pairs = A \E[hh^\T] A^\T = A(I-\Lambda)^{-1} \E[\eta \eta^\T] (I-\Lambda)^{-\T} A^\T\,.
\end{eqnarray}
Define $M := A(I-\Lambda)^{-1} \in \R^{n \times k}$. Then
\begin{eqnarray}
\begin{split}
\Pairs = M \E[\eta\eta^\T] M^\T 
= M \diag(\sigma^2_{\eta(1)}, \dotsc, \sigma^2_{\eta(k)}) M^\T \,.
\end{split}
\end{eqnarray}

\aj{Given that $A$ satisfies the rank condition, it is immediate to see that $M \diag(\sigma_{\eta(1)}, \dotsc, \sigma^2_{\eta(k)})$ also satisfies the rank condition. Therefore, applying Lemma~\ref{lem:block-diagonal}, we can decompose $\Pairs$ into its low-rank part ($L_\Pairs$) and its diagonal part ($D_\Pairs$), where
\begin{align}
\Pairs &= M \diag(\sigma^2_{\eta_1}, \dotsc, \sigma^2_{\eta_k}) M^\T,\\
D_\Pairs &=   \diag(\sigma^2_{\veps_1}, \dotsc, \sigma^2_{\veps_1}).
\end{align}
}

Since $A$ has full column rank, $U^\T \Pairs U \in \R^{k \times k}$ also has full rank; hence, the whitening step (Part 1 in $\AlgII$) is possible.  We have
\begin{align*}
I = W^\T \Pairs W = W^\T M \diag(\sigma^2_{\eta(1)}, \dotsc, \sigma^2_{\eta(k)}) M^\T W.
\end{align*}
Therefore, the matrix $N := W^\T M \diag(\sigma_{\eta(1)}, \dotsc, \sigma_{\eta(k)}) \in \R^{k \times k}$ is an orthogonal matrix.

\begin{lemma}\label{lem:triple}
We have
\begin{eqnarray}
\Triples(\zeta) = M \diag(\mu_{\eta(1)},\dotsc, \mu_{\eta(k)}) \diag(M^\T \zeta) M^\T.
\end{eqnarray}
\end{lemma}

Lemma~\ref{lem:triple} is proved in Appendix~\ref{app:triple}.

\aj{Applying Lemma~\ref{lem:block-diagonal} again, we decompose $\Triples(W\theta)$ into its diagonal and low-rank parts.
\begin{align}
\Triples(W\theta) &= M \diag(\mu_{\eta(1)},\dotsc, \mu_{\eta(k)}) \diag(M^\T W \theta) M^\T,\\
D_{\Triples} &=   \diag(\mu_{\veps(1)}(W\theta)(1), \dotsc, \mu_{\veps(n)} (W\theta)(n)).
\end{align}
}
Now, observe that
\begin{eqnarray}
\begin{split}
&W^\T \Triples(W\theta) W =\\
&W^\T M \diag(\mu_{\eta(1)},\dotsc, \mu_{\eta(k)}) \diag(M^\T W \theta) M^\T W =\\
&N \diag(\sigma_{\eta(1)}, \dotsc, \sigma_{\eta(k)})^{-1} \diag(\mu_{\eta(1)},\dotsc, \mu_{\eta(k)}) \diag(M^\T W \theta)
\diag(\sigma_{\eta(1)}, \dotsc, \sigma_{\eta(k)})^{-1} N^\T
\end{split}
\end{eqnarray}
Since $N$ is an orthogonal matrix, the above is an SVD of $W^\T \Triples(W\theta) W$, and $N_1, \dotsc, N_k$ are singular vectors, where $N_i$ denotes the $i$-th column of $N$. Note that $N_i = \sigma_{\eta(i)} W^\T M_i $ for $i \in [k]$.

A key observation is that an SVD uniquely determines all singular vectors (up to sign) which have distinct singular values. Following a similar approach to~\cite{Anima-SVD-12}, we sample $\theta$ uniformly at random over the sphere in $\R^k$ to ensure that all the singular values of $W^\T \Triples(W\theta) W$ are distinct. Define
\begin{eqnarray}
D := \diag(\sigma_{\eta(1)}, \dotsc, \sigma_{\eta(k)})^{-1} \diag(\mu_{\eta(1)},\dotsc, \mu_{\eta(k)}) \diag(M^\T W \theta)
\diag(\sigma_{\eta(1)}, \dotsc, \sigma_{\eta(k)})^{-1}.
\end{eqnarray}
Note that the diagonal of the matrix $D$ is the following vector:
\begin{align*}
&\diag(\sigma_{\eta(1)}, \dotsc, \sigma_{\eta(k)})^{-1} \diag(\mu_{\eta(1)},\dotsc, \mu_{\eta(k)})
\diag(\sigma_{\eta(1)}, \dotsc, \sigma_{\eta(k)})^{-1} M^\T W \theta\\
& = \diag(\sigma_{\eta(1)}, \dotsc, \sigma_{\eta(k)})^{-1} \diag(\mu_{\eta(1)},\dotsc, \mu_{\eta(k)})
\diag(\sigma_{\eta(1)}, \dotsc, \sigma_{\eta(k)})^{-2} N^\T\theta.
\end{align*}

Since $\theta$ is sampled uniformly over the sphere,  and $N$ is a rotation matrix, the distribution of $N^\T \theta$ is also uniform over the sphere. Consequently, all the singular values of $W^\T \Triples(W\theta) W$ are non-zero and distinct. Therefore, the set $\Omega$ (in step $(5)$ of the algorithm) is given by
\begin{eqnarray*}
\Omega = \{\sigma_{\eta(i)} W^\T M_i\}_{i=1}^k.
\end{eqnarray*}

The columns of matrix $S$, defined in step $(6)$ of the algorithm, are then
\begin{align*}
\{(W^+) ^\T \omega : \omega \in \Omega\} &= \{W(W^\T W)^{-1} \sigma_{\eta(i)} W^\T M_i\}_{i=1}^k \\
&= \{W(W^\T W)^{-1} W^\T \sigma_{\eta(i)} M_i\}_{i=1}^k =  \{\sigma_{\eta(i)} M_i\}_{i=1}^k,
\end{align*}
where the last step holds since $W(W^\T W)^{-1} W^\T $ is a projection and Range($W$) = Range($U$) = Range($\Pairs$) = Range($M$). Hence, there exists permutation $\Pi_1$, such that
\begin{eqnarray*}
S = M \diag(\sigma_{\eta(1)}, \dotsc, \sigma_{\eta(k)})\Pi_1 = A(I-\Lambda)^{-1} \diag(\sigma_{\eta(1)}, \dotsc, \sigma_{\eta(k)})\Pi_1.
\end{eqnarray*}

Note that $\Col(S) = \Col(A)$. As demonstrated in the proof of Theorem~\ref{thm:core}, we can identify all the columns of $A$, as $A$ satisfies the graph expansion and the parameter genericity  property. Moreover, under the assumptions of Theorem~\ref{thm:alg1}, $\AlgI(\Pairs)$ returns all columns of $A$. Therefore, we can recover $\widehat{A} = A \Pi_2$, for a permutation matrix $\Pi_2 \in \R^{k \times k}$. Let $\widehat{B}$ be a left inverse of $\widehat{A}$. Then
\begin{eqnarray*}
C := \widehat{B} S = \widehat{B} A(I-\Lambda)^{-1} \diag(\sigma_{\eta(1)}, \dotsc, \sigma_{\eta(k)})\Pi_1 = \Pi_2^{-1} (I-\Lambda)^{-1} \diag(\sigma_{\eta(1)}, \dotsc, \sigma_{\eta(k)})\Pi_1.
\end{eqnarray*}

Consider a topological ordering of the induced DAG on the hidden nodes. In such an ordering, for every directed edge $(j,i)$, we have $j < i$. Hence, $\Lambda$ would be a lower triangular matrix in a topological ordering. We proceed by reordering the rows and the columns of $C$ to get a lower triangular matrix. This may be done in many different ways but we show that all possible permutations that make $C$ lower triangular correspond to different topological orderings of the same DAG. Therefore, we can choose any such permuted version of $C$, call it $\tilde{C}$. Then there exists a topological ordering with corresponding matrix $\Lambda$, such that, $(I - \Lambda)^{-1} \diag(\sigma_{\eta(1)}, \dotsc, \sigma_{\eta(k)}) = \tilde{C}$ and thus $\Lambda  = I - \diag(\tilde{C})\tilde{C}^{-1}$.

Let ${\sf R}_1$ denote the set of rows in $C$ with exactly one non-zero entry. In any lower triangular version of $C$, the rows in ${\sf R}_1$ should appear on top. Furthermore, their non-zero entries should appear in the first ${\sf R}_1$ columns. Note that rows in ${\sf R}_1$ correspond to hidden nodes with no parent. Obviously, any ordering of them with labels $1,\dotsc, |{\sf R}_1|$ is faithful to topological orderings. Now, we can remove these nodes from the DAG (equivalently eliminate the ${\sf R}_1$ columns and rows from $C$) and repeat the same argument. Therefore, different permuted versions of $C$ which are lower triangular correspond to different topological orderings of the DAG. This completes the proof.

\subsection{Proof of Theorem~\ref{thm:multi}}
\label{subsec:multi-proof}
We identify the matrices $A_i$ (up to permutation of their columns) in a sequential manner.
Let $h_{L_i}$ denote the vector formed by the hidden variables in level $L_i$, for $i \in [m-1]$.
Also, let $\veps_{L_i}$ be the noise vector formed by the noise variables associated to the hidden nodes in level $L_i$, for $i \in [m-1]$.
Write
\begin{eqnarray}
\Sigma = A_{m-1} \E[h_{L_{m-1}}h_{L_{m-1}}^\T] A_{m-1}^\T
+ \E[\veps_{L_{m-1}}\veps_{L_{m-1}}^\T].
\end{eqnarray}

Applying Lemma~\ref{lem:block-diagonal}, we can decompose $\Sigma$ into its low-rank and diagonal parts. 
Therefore we have access to $A_{m-1} \E[h_{L_{m-1}}h_{L_{m-1}}^\T]A_{m-1}^\T$. 
 
By a similar argument used in the proof of Theorem~\ref{thm:core}, we can identify the columns of $A_{m-1}$. Equivalently, we recover $\widehat{A}_{m-1} = A_{m-1} \Pi_{m-1}$ for some permutation matrix $\Pi_{m-1}$. Let $\widehat{B}_{m-1}$ be a left inverse of $\widehat{A}_{m-1}$. Now, notice that
\begin{eqnarray}
\widehat{B}_{m-1} A_{m-1} \E[h_{L_{m-1}}h_{L_{m-1}}^\T]A_{m-1}^\T \widehat{B}^\T_{m-1} = \Pi_{m-1}^{-1} \E[h_{L_{m-1}}h_{L_{m-1}}^\T]\Pi_{m-1}^{-\T}.
\end{eqnarray}

In words, we can recover the second order moment of the hidden variables in level $L_{m-1}$, up to a permutation of the nodes within this level. Using the same technique sequentially, we can recover all the columns of $A_i$ for $i \in [m-1]$ and thus the entire model is identifiable up to permutation of hidden nodes within each level.
\section{Proof of Remark~\ref{rem:genericity}}
\label{app:genericity}

Let $\tilde{M} := M + Z$. We first establish some definitions.

\begin{definition}\label{def:fully dense}
We call a vector \emph{fully dense} if all of its entries are non-zero.
\end{definition}

\begin{definition}\label{def:NSP}
We say a matrix has the \emph{Null Space Property (NSP)} if its null space
does not contain any fully dense vector.
\end{definition}

\begin{claim} \label{claim:nullspace}
Fix any $S \subseteq [k]$ with $|S| \geq 2$, and  set $R := \Neigh_M(S)$.
Let $\tilde{C}$ be a $|S| \times |S|$ submatrix of $\tilde{M}_{R,S}$.
Then $\Pr(\text{$\tilde{C}$ has the NSP}) = 1$.
\end{claim}

Now, we are ready to prove Remark~\ref{rem:genericity}.

\begin{proof}[Proof (Remark~\ref{rem:genericity})]

It follows from Claim~\ref{claim:nullspace} that, with probability one, the
following event holds: for every $S \subseteq [k]$ with $|S| \geq 2$, and
every $|S| \times |S|$ submatrix $\tilde{C}$ of $\tilde{M}_{R,S}$, $\tilde{C}$ has
the NSP.
Henceforth condition on this event.

Now fix $v \in \R^k$ with $\|v\|_0 \geq 2$.
Let $S := \supp(v)$, $R := \Neigh_{M}(S)$ and $B := \tilde{M}_{R,S}$.
Furthermore, let $u \in (\R \setminus \{0\})^{|S|}$ be the restriction
of vector $v$ to $S$; observe that $u$ is fully dense.
It is clear that $\|\tilde{M}v\|_0 = \|Bu\|_0$, so we need to show that
\begin{eqnarray}
\|Bu\|_0 > |R| - |S| .
\end{eqnarray}
Suppose for sake of contradiction that $Bu$ has at most $|R| - |S|$
non-zero entries.
Then there is a subset of $|S|$ entries on which $Bu$ is zero.
This corresponds to a $|S| \times |S|$ submatrix of $B$ which contains
$u$ in its null space, which means that this submatrix does not have
the NSP---a contradiction.
Therefore we conclude that $Bu$ must have more than $|R| - |S|$
non-zero entries.
\end{proof}

\begin{proof}[Proof (Claim~\ref{claim:nullspace})]
Let $s = |S|$ and let $\tilde{C} = [\tilde{c}_1|\tilde{c}_2|\dotsb|\tilde{c}_s]^\T$, where $\tilde{c}_i^\T$ is the
$i$-th row of $\tilde{C}$.
Also, let $C := [c_1|{c}_2|\dotsb|{c}_s]^\T$ and $W :=
[w_1|w_2|\dotsb|w_s]^\T$ be the corresponding submatrices of $M$ and
$Z$, respectively.
For each $i \in [s]$, denote by $\mathcal{N}_i$ the null space of the
matrix $\tilde{C}_i = [\tilde{c}_1|\tilde{c}_2|\dotsb|\tilde{c}_i]^\T$.
Finally let $\mathcal{N}_0 = \R^s$.
Then, $\mathcal{N}_0 \supseteq \mathcal{N}_1\supseteq \dotsb \supseteq
\mathcal{N}_s$.
We need to show that, with probability one, $\mathcal{N}_s$ does not
contain any fully dense vector.

If one of $\mathcal{N}_i$ does not contain any full dense vector then we
are done.
Suppose that $\mathcal{N}_i$ contains some fully dense vector $v$.
Since $C$ is a submatrix of $M_{R,S}$, every row
${c}_{i+1}^\T$ of ${C}$ contains at least one non-zero entry.
Therefore
\begin{align*}
v^\T \tilde{c}_{i+1}
& = \sum_{j \in [s]} v(j) \tilde{c}_{i+1}(j)
\\
& = \sum_{j \in [s] : {c}_{i+1}(j) \neq 0}
v(j) ({c}_{i+1}(j) + w_{i+1}(j))
\end{align*}
where $\{ w_{i+1}(j) : j \in [s] \ \text{s.t.} \ {c}_{i+1}(j) \neq 0
\}$ are independent random variables (from $Z$).
Moreover, they are of $\tilde{c}_1, \dotsc,\tilde{c}_i$ and thus of $v$.
By assumption on the distribution of the $w_{i+1}(j)$,
\begin{eqnarray}
\P\Biggl[
v \in \mathcal{N}_{i+1}
\bigg| \tilde{c}_1,\tilde{c}_2,\dotsc,\tilde{c}_i \Biggr]
=
\P\Biggl[
\sum_{j \in [s] : {c}_{i+1}(j) \neq 0}
v(j) ({c}_{i+1}(j) + w_{i+1}(j)) = 0
\bigg| \tilde{c}_1,\tilde{c}_2,\dotsc,\tilde{c}_i \Biggr]
= 0 .
\end{eqnarray}
Consequently,
\begin{eqnarray}
\P\Biggl[
\dim(\mathcal{N}_{i+1}) < \dim(\mathcal{N}_i) \bigg| \tilde{c}_1,\tilde{c}_2,\dotsc,\tilde{c}_i
\biggr] = 1
\end{eqnarray}
for all $i = 0,\dotsc, s-1$.
As a result, with probability one, $\dim(\mathcal{N}_s) = 0$.
\end{proof}

\section{Proof of Lemma~\ref{lem:triple}}
\label{app:triple}
\begin{eqnarray}
\begin{split}
\Triples(\zeta) &= \E[x_1x_2^\T \langle \zeta, x_3\rangle] = \E[\E[x_1x_2^\T \langle \zeta, x_3\rangle | h]]\\
& = \E [\E[x_1|h] \E[x_2|h]^\T \langle \zeta, \E[x_3|h] \rangle ]\\
&= \E [Ahh^\T A^\T \langle \zeta, Ah \rangle]\\
&= \E[M\eta \eta^\T M^\T \langle \zeta, M \eta \rangle]\\
& = M\E[\eta \eta^\T \langle \eta, M^\T \zeta \rangle] M^\T.
\end{split}
\end{eqnarray}

The proof is completed by showing that for any deterministic vector $v \in \R^k$, and any random vector $z  = (z(1), \dotsc, z(k))$ with zero mean independent entries, we have
\begin{eqnarray}
\E[z z^\T \langle z,v \rangle] = \diag(v) \diag(\mu_{z(1)}, \dotsc, \mu_{z(n)}).
\end{eqnarray}
We compute the diagonal and off-diagonal entries separately.
\begin{align}
\E[z(i) z(i) \langle v, z\rangle] = v(i) \E[z(i)^3] + \sum_{k \neq i} v(k) \sigma_{z(i)}^2 \E[z(k)] = v(i)\mu_{z(i)}.
\end{align}
For $j \neq i$
\begin{align*}
\E[z(i) z(j) \langle v,z\rangle] &= \E[z(i) z(j) \sum_{k} v(k) z(k)]\\
&= v(i) \sigma_{z(i)}^2 \E[z(j)] + v(j) \sigma_{z(j)}^2 \E[z(i)] + \sum_{k \neq i,j} v(k) \E[z(i)]\E[z(j)]\E[z(k)]
 = 0.
\end{align*}

\section{Proof of Remark~\ref{rem:eff-second}}
\label{app:eff-second}
Write
\begin{eqnarray}
\Pairs = A \E[hh^\T] A^\T\,.
\end{eqnarray}

By Theorem~\ref{thm:core}, we can identify the columns of $A$, \emph{i.e.}, we can recover $\widehat{A} = A \Pi_1$ for some permutation matrix $\Pi_1$. 
 Let $\widehat{B} \in \R^{k \times n}$ be a left inverse of $\widehat{A}$. Then,
\begin{eqnarray}
\widehat{B}A \E[hh^\T]A^\T \widehat{B}^\T = \Pi_1^{-1} \E[hh^\T] \Pi_1^{-\T}.
\end{eqnarray}

Therefore, we have the second order moment of the hidden nodes (in some ordering of the nodes). Now consider $k$ hidden nodes corresponding to the row (and columns of ) $\Pi_1^{-1} \E[hh^\T] \Pi_1^{-\T}$. Label these nodes with $1, \dotsc, k$. Using the oracle we can find a permutation $\pi_2$ which puts the hidden nodes in a topological ordering. Let $\Pi_2$ be the corresponding permutation matrix to $\pi_2$. Then $\widetilde{\Pairs}:= \Pi_2 \Pi_1^{-1} \E[hh^\T] \Pi_1^{-\T} \Pi_2^\T$ is the second order moment of the hidden nodes in some topological ordering. By definition of a topological ordering, it is immediate to see that the coefficient matrix $\Lambda$ is lower triangular in a topological ordering of the hidden nodes. Therefore, we can write
\begin{eqnarray}
\widetilde{\Pairs} = (I - \Lambda)^{-1} \E[\eta \eta^\T] (I-\Lambda)^{-\T},
\end{eqnarray}
where $\eta$ is the vector formed by the noise variables $\eta(i)$ (in the corresponding topological ordering) and $\Lambda \in \R^{k \times k}$ is a lower triangular matrix with all diagonal entries equal to zero.  Therefore,
\begin{eqnarray}
\widetilde{\Pairs}^{1/2} = (I - \Lambda)^{-1} \diag(\sigma_{\eta(1)}, \dotsc, \sigma_{\eta(k)}) Q,
\end{eqnarray}
 for some rotation $Q \in \R^{k \times k}$. Notice that $L: = (I - \Lambda)^{-1} \diag(\sigma_{\eta(1)}, \dotsc, \sigma_{\eta(k)})$ is a lower triangular matrix with diagonal entries $\sigma_{\eta(1)} , \dotsc, \sigma_{\eta(k)}$ which are all positive. Hence, using the LQ decomposition of $\widetilde{\Pairs}^{1/2}$, we can recover $L$. (Recall that the LQ factorization is unique if we require that the diagonal entries of the lower triangular part are positive).

Finally, $\diag(L) = \diag((I-\Lambda)^{-1}) \diag(\sigma_{\eta(1)}, \dotsc, \sigma_{\eta(k)}) = \diag(\sigma_{\eta(1)}, \dotsc, \sigma_{\eta(k)})$. Therefore, $\Lambda = I - \diag(L) L^{-1}$.  The result follows.
\section{Proof of Lemma~\ref{lem:block-diagonal}}
\label{app:block-diagonal}
For each $I \in \Part$, let $U_I, V_I \in \R^{|I| \times k}$ be any matrices such that $U_I^\T A_I$ and $V_I^\T B$ are invertible. Then for any distinct $I, J, K \in \Part$,
\begin{align}
A_I B_I^\T &= A_I(B_J^\T V_J)(B_J^\T V_J)^{-1} (U_K^\T A_K)^{-1} (U_K^\T A_K) B_I^\T\nonumber\\
&= A_IB_J^\T V_J(U_K^\T A_KB_J^\T V_J)^{-1} U_K^\T A_K B_I^\T.\label{eqn:dummy1}
\end{align}

Notice that for any distinct $I,J \in \Part$, $C_{I,J} = A_I B_J^\T$. Since $A_I$ and $B_J$ have rank $k$, so does $C_{I,J}$. Let $U_I \in \R^{|I| \times k}$ and $V_J \in \R^{|J|\times k}$ be respectively the matrices of left and right singular vectors of $C_{I,J}$ (corresponding to non-zero singular values). Since $U_I$ and $A_I$ have the same range, it follows that $U_I^\T A_I$ is invertible. Similarly $V_J^\T B_J$ is invertible. Using identity~\eqref{eqn:dummy1}, we obtain
\begin{eqnarray}
A_IB_I^\T = C_{I,J} V_J(U_K^\T\Pairs_{K,J}V_J)^{-1} U_K^\T C_{K,I},
\end{eqnarray}
for any distinct $I, J, K \in \Part$. Therefore $D$ can be determined as $D_{I,I} = C_{I,I} - A_IB_I^\T$ for $I \in \Part$ and $L = AB^\T$ is subsequently determined as $L = C - D$.

\section{Proof of Lemma~\ref{lemma:partitioning}}
\label{app:partitioning}
Let $A = USV^\T$ be a thin singular value decomposition of $A$, where $U \in \R^{n\times k}$ has orthonormal columns,
$S = \diag(\sigma_1(A), \dotsc, \sigma_k(A))$, and $V \in \R^{k \times k}$ is an orthogonal matrix.
Fix a partition index $v \in [\ell]$.
Let $z_1, z_2, \dotsc, z_n \in \{0,1\}$ be independent indicator random
variables such that $z_i = 1$ iff row $i$ of $A$ is included in
$A_v$.
Note that
\begin{eqnarray}
\begin{split}
A_v^\T A_v &= A^\T \diag(z_1,z_2,\dotsc,z_n) A\\
&=\sum_{i=1}^n z_i A^\T e_ie_i^\T A = V S (\sum_{i=1}^n z_i U^\T e_i e_i^\T
U ) SV^\T.
\end{split}
\end{eqnarray}
Therefore
\begin{eqnarray}\label{eqn:lambda-min}
\sigma_k(A_v)^2 = \lambda_{\min}(A_v^\T A_v)
\geq \lambda_{\min}(S)^2 \cdot \lambda_{\min}( \sum_{i=1}^n z_i U^\T
e_ie_i^\T U) = \sigma_k(A)^2 \cdot \lambda_{\min} (\sum_{i=1}^n X_i),
\end{eqnarray}
where $X_i :=  z_i U^\T e_ie_i^\T U \in \R^{k \times k}$.
Notice that $0 \preceq X_i $ and
\begin{eqnarray}
\lambda_{\max}(X_i) \le \|U^\T e_i\|_2^2 \le \frac{k}{n} c_A.
\end{eqnarray}
Moreover,
\begin{eqnarray}
\sum_{i=1}^n \E X_i = \sum_{i=1}^n \P(z_i = 1) U^\T e_ie_i^\T U = \frac{1}{\ell} U^\T U = \frac{1}{\ell}I.
\end{eqnarray}
By Lemma~\ref{lemma:matrix-chernoff},
\begin{eqnarray}
\P\bigg\{ \lambda_{\min}(\sum_{i=1}^d X_i) \leq \frac{1}{4\ell} \bigg\} \leq k \cdot e^{-(3/4)^2/(2\ell c_A k/n)} \leq \delta/\ell,
\end{eqnarray}
 where the last inequality follows from the assumption on $c_A$.
Therefore by Eq.~\eqref{eqn:lambda-min}, $\sigma_k(A_v) \geq
\sigma_k(A) / (2\sqrt{\ell})$,
with probability at least $1-\delta/\ell$. A union bound over all $v \in [\ell]$ completes the proof.

\begin{lemma}[Matrix Chernoff bound~\cite{Tropp12}] \label{lemma:matrix-chernoff}
Consider a finite sequence $\{X_i\}$ of independent and symmetric $k \times k$
random matrices such that $0 \preceq X_i$  and $\lambda_{\max}(X_i) \le r$ almost surely.
Define $\mu_{\min}:= \lambda_{\min}(\sum_{i} \E X_i)$.
For any $\epsilon \in [0,1]$, we have
\begin{eqnarray}
\P\bigg \{ \lambda_{\min}\biggl(\sum_{i} X_i\biggr) \leq
(1-\epsilon) \mu_{\min}\biggr\} \leq k \cdot e^{-\epsilon^2 \mu_{\min} / (2r)}.
\end{eqnarray}
\end{lemma}
%
%
%
\newpage
\section{$\AlgIproj$}
\label{app:AlgIproj}
Below is the slight variant of $\AlgI$ used in numerical experiments.

\begin{algorithm}[]
\caption*{$\AlgIproj$: Learning the topic-word matrix from 
pairwise correlations, using iterative \mbox{projections}.}
\begin{algorithmic}[1]

\REQUIRE Second order moment of the observed variables ($\Pairs$).

\ENSURE Columns of $A$ up to permutation.

\STATE Find a partition $\Part$ of $[n]$ such that $|\Part| = 3$ and $\rank(\Pairs_{I,J}) = k$ for distinct $I,J \in \Part$.

\STATE Let $L$ be the low-rank part returned by $\Subroutine(\Pairs, \Part)$.

\STATE Set $\mathcal{S} = \{0\} \subset \R^n$.

\FOR{each $i\in [k]$}

\FOR{each $j\in [n]$}

\STATE Solve the optimization problem
\[ \min_{w} \,\, \|L^{1/2} w\|_1 \quad \quad \text{subject to } (e_j^\T
L^{1/2}) {\rm P}_{\mathcal{S}^\perp} w = 1 . \]
Denote the solution by $w_{ij}$.
\ENDFOR

\STATE Set $w_i = \arg\min_{w_{i1},\dotsc, w_{in}} \|L^{1/2} w\|_0$, breaking ties arbitrarily.

\STATE $\mathcal{S} = \mathcal{S} \cup \{w_i\}$.
\ENDFOR

\RETURN $\Big\{\frac{L^{1/2} w_i}{\|L^{1/2} w_i\|}\Big\}_{i=1}^k$.

\end{algorithmic}
\end{algorithm}


%% file: draft3.bbl
\begin{thebibliography}{10}

\bibitem{Abbeel2006learning}
P.~Abbeel, D.~Koller, and A.~Ng.
\newblock Learning factor graphs in polynomial time and sample complexity.
\newblock {\em Journal of Machine Learning Research}, 7:1743--1788, 2006.

\bibitem{ali2005towards}
R.~Ali, T.~Richardson, P.~Spirtes, and J.~Zhang.
\newblock Towards characterizing {M}arkov equivalence classes for directed
  acyclic graphs with latent variables.
\newblock In {\em Proceedings of the 21th Conference on Uncertainty in
  Artificial Intelligence}, 2005.

\bibitem{Anima-spectral-11}
A.~Anandkumar, K.~Chaudhuri, D.~Hsu, S.~M. Kakade, L.~Song, and T.~Zhang.
\newblock Spectral methods for learning multivariate latent tree structure.
\newblock In {\em Advances in Neural Information Processing Systems}, 2011.

\bibitem{Anima-SVD-12}
A.~Anandkumar, D.~P. Foster, D.~Hsu, S.~M. Kakade, and Y.-K. Liu.
\newblock {Two SVDs Suffice: Spectral decompositions for probabilistic topic
  modeling and latent Dirichlet allocation}.
\newblock arXiv:1204.6703v3, 2012.

\bibitem{AGHKT}
A.~Anandkumar, R.~Ge, D.~Hsu, S.~M. Kakade, and T.~Telgarsky.
\newblock Tensor decompositions for learning latent variable models.
\newblock arXiv:1210.7559, 2012.

\bibitem{AHK12}
A.~Anandkumar, D.~Hsu, and S.~Kakade.
\newblock A method of moments for mixture models and hidden {M}arkov models.
\newblock In {\em COLT}, 2012.

\bibitem{AnandkumarTanWillsky:Ising11}
A.~Anandkumar, V.~Y.~F. Tan, F.~Huang, and A.~S. Willsky.
\newblock {High-dimensional structure learning of Ising models: local
  separation criterion}.
\newblock {\em Annals of Statistics}, 40(3):1346--1375, 2012.

\bibitem{Anandkumar:girth12}
A.~Anandkumar and R.~Valluvan.
\newblock Learning loopy graphical models with latent variables: Efficient
  methods and guarantees.
\newblock arXiv:1203.3887, 2012.

\bibitem{AroraICML}
S.~Arora, R.~Ge, Y.~Halpern, D.~M. Mimno, A.~Moitra, D.~Sontag, Y.~Wu, and
  M.~Zhu.
\newblock A practical algorithm for topic modeling with provable guarantees.
\newblock {\em ArXiv 1212.4777}, 2012.

\bibitem{arora2012learning}
S.~Arora, R.~Ge, and A.~Moitra.
\newblock Learning topic models---going beyond svd.
\newblock In {\em Symposium on Theory of Computing}, 2012.

\bibitem{austin2008exchangeable}
T.~Austin.
\newblock On exchangeable random variables and the statistics of large graphs
  and hypergraphs.
\newblock {\em Probab. Survey}, 5:80--145, 2008.

\bibitem{Awokuse-03}
T.~O. Awokuse and D.~A. Bessler.
\newblock Vector autoregressions, policy analysis, and directed acyclic graphs:
  An application to the {U.S.} economy.
\newblock {\em {Journal of Applied Economics}}, VI:1--24, 2003.

\bibitem{Bagozzi-80}
R.~Bagozzi.
\newblock {\em {Causal models in marketing}}.
\newblock Theories in marketing series. Wiley, New York, 1980.

\bibitem{berinde2008combining}
R.~Berinde, A.~C. Gilbert, P.~Indyk, H.~Karloff, and M.~J. Strauss.
\newblock Combining geometry and combinatorics: A unified approach to sparse
  signal recovery.
\newblock In {\em Communication, Control, and Computing, 2008 46th Annual
  Allerton Conference on}, pages 798--805. IEEE, 2008.

\bibitem{blei2007correlated}
D.~Blei and J.~Lafferty.
\newblock A correlated topic model of science.
\newblock {\em Annals of Applied Statistics}, pages 17--35, 2007.

\bibitem{blei2003latent}
D.~Blei, A.~Ng, and M.~Jordan.
\newblock Latent {D}irichlet allocation.
\newblock {\em Journal of Machine Learning Research}, 3:993--1022, 2003.

\bibitem{blei2012probabilistic}
D.~M. Blei.
\newblock Probabilistic topic models.
\newblock {\em Communications of the ACM}, 55(4):77--84, 2012.

\bibitem{Bollen-89}
K.~A. Bollen.
\newblock {\em {Structural Equations with Latent Variables}}.
\newblock Wiley, New York, 1989.

\bibitem{Boutilier:UAI96}
C.~Boutilier, N.~Friedman, M.~Goldszmidt, and D.~Koller.
\newblock Context-specific independence in {B}ayesian networks.
\newblock In {\em Proceedings of the 12th Annual Conference on Uncertainty in
  Artificial Intelligence}, 1996.

\bibitem{Bresler&etal:Rand}
G.~Bresler, E.~Mossel, and A.~Sly.
\newblock Reconstruction of {M}arkov random fields from samples: some
  observations and algorithms.
\newblock In {\em Intl. workshop APPROX Approximation, Randomization and
  Combinatorial Optimization}. Springer, 2008.

\bibitem{Chandrasekaran:10latent}
V.~Chandrasekaran, P.~Parrilo, and A.~Willsky.
\newblock Latent variable graphical model selection via convex optimization.
\newblock {\em Annals of Statistics (to appear)}, 2012.

\bibitem{Venkat11}
V.~Chandrasekaran, S.~Sanghavi, P.~A. Parrilo, and A.~S. Willsky.
\newblock Rank-sparsity incoherence for matrix decomposition.
\newblock {\em SIAM Journal on Optimization}, 21(2):572--596, 2011.

\bibitem{Chickering-03}
D.~M. Chickering.
\newblock Optimal structure identification with greedy search.
\newblock {\em Journal of Machine Learning Research}, 3:507--554, 2003.

\bibitem{Choi&etal:10JMLR}
M.~Choi, V.~Tan, A.~Anandkumar, and A.~Willsky.
\newblock Learning latent tree graphical models.
\newblock {\em Journal of Machine Learning Research}, 12:1771--1812, 2011.

\bibitem{Choi-10}
M.~J. Choi, J.~J. Lim, A.~Torralba, and A.~S. Willsky.
\newblock Exploiting hierarchical context on a large database of object
  categories.
\newblock In {\em IEEE Conference on Computer Vision and Pattern Recognition},
  2010.

\bibitem{Chow&Liu:68IT}
C.~Chow and C.~Liu.
\newblock {Approximating discrete probability distributions with dependence
  trees}.
\newblock {\em IEEE Tran. on Information Theory}, 14(3):462--467, 1968.

\bibitem{daskalakis06}
C.~Daskalakis, E.~Mossel, and S.~Roch.
\newblock Optimal phylogenetic reconstruction.
\newblock In {\em Proceedings of the Thirty-Eighth Annual ACM Symposium on
  Theory of Computing}, 2006.

\bibitem{Daubechies-IST}
I.~Daubechies, M.~Defrise, and C.~De~Mol.
\newblock {An iterative thresholding algorithm for linear inverse problems with
  a sparsity constraint}.
\newblock {\em Comm. Pure Appl. Math.}, 57(11):1413--1457, 2004.

\bibitem{Durbin-98}
R.~Durbin, S.~Eddy, A.~Krogh, and G.~Mitchison.
\newblock {\em Biological Sequence Analysis: Probabilistic Models of Proteins
  and Nucleic Acids}.
\newblock Cambridge University Press, 1998.

\bibitem{erdos99}
P.~L. Erd\"{o}s, L.~A. Sz\'{e}kely, M.~A. Steel, and T.~J. Warnow.
\newblock A few logs suffice to build (almost) all trees: Part {I}.
\newblock {\em Random Structures and Algorithms}, 14:153--184, 1999.

\bibitem{Figueiredo-gradient}
M.~A.~T. Figueiredo, R.~D. Nowak, and S.~J. Wright.
\newblock {Gradient Projection for Sparse Reconstruction: Application to
  Compressed Sensing and Other Inverse Problems}.
\newblock {\em IEEE Journal of Selected Topics in Signal Processing},
  1(4):586--597, 2007.

\bibitem{gottlieb2010matrix}
L.-A. Gottlieb and T.~Neylon.
\newblock Matrix sparsification and the sparse null space problem.
\newblock {\em Approximation, Randomization, and Combinatorial Optimization.
  Algorithms and Techniques}, pages 205--218, 2010.

\bibitem{Haavelmo-43}
T.~Haavelmo.
\newblock The statistical implications of a system of simultaneous equations.
\newblock {\em {Econometrica}}, 11:1--12, 1943.

\bibitem{hauser2011characterization}
A.~Hauser and P.~B{\"u}hlmann.
\newblock Characterization and greedy learning of interventional {M}arkov
  equivalence classes of directed acyclic graphs.
\newblock {\em Journal of Machine Learning Research}, 13:2409--2464, 2012.

\bibitem{Hoyer-09}
P.~O. Hoyer, D.~Janzing, J.~M. Mooij, J.~Peters, and B.~Sch\"olkopf.
\newblock {Nonlinear causal discovery with additive noise models}.
\newblock In {\em Advances in Neural Information Processing Systems}, 2009.

\bibitem{HKZ11}
D.~Hsu, S.~M. Kakade, and T.~Zhang.
\newblock Robust matrix decomposition with sparse corruptions.
\newblock {\em IEEE Transactions on Information Theory}, 57(11):7221--7234,
  2011.

\bibitem{ICAbook}
A.~Hyv{\"a}rinen, J.~Karhunen, and E.~Oja.
\newblock {\em Independent Component Analysis}.
\newblock Wiley Interscience, 2001.

\bibitem{Jalali:greedy}
A.~Jalali, C.~Johnson, and P.~Ravikumar.
\newblock On learning discrete graphical models using greedy methods.
\newblock In {\em Proc. of NIPS}, 2011.

\bibitem{Kim-gradient}
S.-J. Kim, K.~Koh, M.~Lustig, S.~Boyd, and D.~Gorinevsky.
\newblock {An Interior-Point Method for Large-Scale L1-Regularized Least
  Squares}.
\newblock {\em IEEE Journal of Selected Topics in Signal Processing},
  1(4):606--617, Dec. 2007.

\bibitem{Kohn-82}
M.~Kohn and C.~Schooler.
\newblock {Job conditions and personality: A longitudinal assessment of their
  reciprocal effects}.
\newblock {\em American Journal of Sociology}, 87(6):1257--1286, 1982.

\bibitem{Koller:SRL07}
D.~Koller, N.~Friedman, L.~Getoor, and B.~Taskar.
\newblock Graphical models in a nutshell.
\newblock In L.~Getoor and B.~Taskar, editors, {\em {Introduction to
  Statistical Relational Learning}}. MIT Press, 2007.

\bibitem{Koller:UAI97}
D.~Koller and A.~Pfeffer.
\newblock Object-oriented {B}ayesian networks.
\newblock In {\em Proceedings of the 13th Annual Conference on Uncertainty in
  Artificial Intelligence}, pages 302--313, 1997.

\bibitem{Lauritzen-96}
S.~Lauritzen.
\newblock {\em {Graphical Models}}.
\newblock Oxford University Press, 1996.

\bibitem{levin2009understanding}
A.~Levin, Y.~Weiss, F.~Durand, and W.~T. Freeman.
\newblock Understanding and evaluating blind deconvolution algorithms.
\newblock In {\em Computer Vision and Pattern Recognition, 2009. CVPR 2009.
  IEEE Conference on}, pages 1964--1971. IEEE, 2009.

\bibitem{li2006pachinko}
W.~Li and A.~McCallum.
\newblock Pachinko allocation: {DAG}-structured mixture models of topic
  correlations.
\newblock In {\em Proc. of Intl. Conf. on Machine learning}, pages 577--584,
  2006.

\bibitem{Mei06}
N.~Meinshausen and P.~B\"{u}hlmann.
\newblock {High dimensional graphs and variable selection with the lasso}.
\newblock {\em Annals of Statistics}, 34(3):1436--1462, 2006.

\bibitem{Nestrov-1}
Y.~Nesterov.
\newblock {A method of solving a convex programming problem with convergence
  rate O(1/$k^2$)}.
\newblock {\em Soviet Mathematics Doklady}, 27(2):372--376, 1983.

\bibitem{Nestrov-2}
Y.~Nesterov.
\newblock {Gradient methods for minimizing composite objective function}, 2007.
\newblock {ECORE Discussion Paper}.

\bibitem{Pearl:book}
J.~Pearl.
\newblock {\em Probabilistic Reasoning in Intelligent Systems---Networks of
  Plausible Inference}.
\newblock Morgan Kaufmann, 1988.

\bibitem{Pearl-soc-98}
J.~Pearl.
\newblock {Graphs, causality, and structural equation models}.
\newblock {\em Sociological Methods and Research}, 27(2):226--284, 1998.

\bibitem{Pearl-00}
J.~Pearl.
\newblock {\em Causality: Models, Reasoning, and Inference}.
\newblock Cambridge University Press, Cambridge, England, 2000.

\bibitem{Peters-12}
J.~Peters and P.~B\"uhlmann.
\newblock Identifiability of {G}aussian structural equation models with same
  error variances.
\newblock arXiv:1205.2536v1, 2012.

\bibitem{Peters-11}
J.~Peters, J.~Mooij, D.~Janzing, and B.~Sch\"olkopf.
\newblock {Identifiability of causal graphs using functional models}.
\newblock In {\em 27th Conference on Uncertainty in Artificial Intelligence},
  2011.

\bibitem{Ravikumar&etal:08Stat}
P.~Ravikumar, M.~Wainwright, and J.~Lafferty.
\newblock High-dimensional {I}sing model selection using $\ell_1$-regularized
  logistic regression.
\newblock {\em Annals of Statistics}, 38(3):1287--1319, 2010.

\bibitem{Roch-12}
S.~Roch and S.~Snir.
\newblock Recovering the tree-like trend of evolution despite extensive lateral
  genetic transfer: a probabilistic analysis.
\newblock In {\em Proceedings of the 16th Annual international conference on
  Research in Computational Molecular Biology}, RECOMB'12, pages 224--238,
  2012.

\bibitem{Saunderson}
J.~Saunderson, V.~Chandrasekaran, P.~A. Parrilo, and A.~S. Willsky.
\newblock Diagonal and low-rank matrix decompositions, correlation matrices,
  and ellipsoid fitting.
\newblock arXiv:1204.1220, 2012.

\bibitem{Shimizu-06}
S.~Shimizu, P.~O. Hoyer, A.~Hyv\"arisen, and A.~Kerminen.
\newblock {A linear non-gaussian acyclic model for causal discovery}.
\newblock {\em Journal of Machine Learning Research}, 7:2003--2030, 2006.

\bibitem{Silva-06}
R.~Silva, R.~Scheines, C.~Glymour, and P.~Spirtes.
\newblock Learning the structure of linear latent variable models.
\newblock {\em Journal of Machine Learning Research}, 7:191--246, 2006.

\bibitem{Spielman-12}
D.~A. Spielman, H.~Wang, and J.~Wright.
\newblock Exact recovery of sparsely-used dictionaries.
\newblock arXiv:1206.5882v1, 2012.

\bibitem{Spirtes-05}
P.~Spirtes.
\newblock {Graphical models, causal inference, and econometric models }.
\newblock {\em Journal of Economic Methodology}, 12:1:1--33, 2005.

\bibitem{Spirtes-00}
P.~Spirtes, C.~Glymour, and R.~Scheines.
\newblock {\em Causation, Prediction, and Search}.
\newblock MIT press, 2nd edition, 2000.

\bibitem{Tropp12}
J.~A. Tropp.
\newblock User-friendly tail bounds for sums of random matrices.
\newblock {\em Foundations of Computational Mathematics}, 12(4):389--434, 2012.

\bibitem{Wainwright-08}
M.~J. Wainwright and M.~I. Jordan.
\newblock Graphical models, exponential families, and variational inference.
\newblock {\em Foundations and Trends in Machine Learning}, 1(1-2):1--305,
  2008.

\bibitem{Wheaton-78}
B.~Wheaton.
\newblock The sociogenesis of psychological disorder.
\newblock {\em {American Sociological Review}}, 43:383--403, 1978.

\bibitem{Zellner-71}
A.~Zellner.
\newblock {\em {Introduction to Bayesian Inference in Econometrics}}.
\newblock New York: John Wiley, 2nd edition, 1971.

\bibitem{zibulevsky2001blind}
M.~Zibulevsky and B.~A. Pearlmutter.
\newblock Blind source separation by sparse decomposition in a signal
  dictionary.
\newblock {\em Neural computation}, 13(4):863--882, 2001.

\end{thebibliography}
